\documentclass[11pt]{article}
\usepackage[utf8]{inputenc}
\usepackage[margin=1in]{geometry}

\usepackage{microtype}
\usepackage{subfigure}
\usepackage{booktabs} 
\usepackage[round]{natbib}

\usepackage{amsmath,amsthm,amsfonts,amssymb,mathdots,array,mathrsfs,bm,bbm,stmaryrd,graphicx,subfigure,xcolor}

\usepackage{algorithm}
\usepackage{algpseudocode}

\let\oldnl\nl% Store \nl in \oldnl
\newcommand{\nonl}{\renewcommand{\nl}{\let\nl\oldnl}}% Remove line number for one line

\usepackage{breakcites}

\usepackage[T1]{fontenc}
\usepackage{enumerate}
\usepackage{inputenc}

\usepackage{graphicx} % more modern
\usepackage{subfigure}

\usepackage{booktabs,balance}
\usepackage{rotating}
\usepackage{boldline}
\usepackage{makecell}
\usepackage{multirow}
\usepackage{balance}
\usepackage{wrapfig}

\usepackage[colorlinks=true,citecolor=blue]{hyperref}

\newtheorem{theorem}{Theorem}[section]

\newtheorem{lemma}[theorem]{Lemma}

\newtheorem{definition}{Definition}[section]

\begin{document}

\title{\bf Differentially Private One Permutation Hashing and\\
Bin-wise Consistent Weighted Sampling}

\author{\\\vspace{0.3in}\\\vspace{0.1in}
  \textbf{Xiaoyun Li, \ \  Ping Li} \\
  LinkedIn Ads\\
  700 Bellevue Way NE, Bellevue, WA 98004, USA\\
  \texttt{ \{xiaoyli, pinli\}@linkedin.com} 
}

\date{\vspace{0.3in}}
\maketitle

\begin{abstract}\vspace{0.2in}
\noindent Minwise hashing (MinHash) is a standard algorithm widely used in the industry, for large-scale search and learning applications with the binary (0/1) Jaccard similarity. One common  use of MinHash is for processing massive n-gram text representations so that practitioners do not have to materialize the original data (which would be prohibitive). Another popular use of MinHash is for building hash tables to enable sub-linear time approximate near neighbor (ANN) search.  MinHash has also been used as a tool for building large-scale machine learning systems. The standard implementation of MinHash requires applying $K$ random permutations. In comparison, the method of one permutation hashing (OPH, developed in 2012), is an  efficient alternative of MinHash which splits the data vectors  into $K$ bins and generates hash values within each bin. OPH is substantially more efficient and also more convenient to use. In many practical scenarios, practitioners can simply store one perfect random permutation as opposed to storing $K$ permutations or resorting to approximate random permutations such as universal hashing. 

\vspace{0.2in}
\noindent In this paper, we combine the differential privacy (DP) with OPH (as well as MinHash), to propose the DP-OPH framework with three variants: DP-OPH-fix, DP-OPH-re and DP-OPH-rand, depending on which  densification strategy is adopted to deal with empty bins in OPH. A detailed roadmap to the algorithm design is presented along with the privacy analysis. An analytical comparison of our proposed DP-OPH methods with the DP minwise hashing (DP-MH)  is provided to justify the advantage of DP-OPH. Experiments on similarity search confirm the merits of  DP-OPH, and guide the choice of the proper variant in different practical scenarios. Our technique is also extended to bin-wise consistent weighted sampling (BCWS) to develop a new DP algorithm called DP-BCWS for non-binary data. Experiments on classification tasks (logistic regression and neural nets) demonstrate that DP-BCWS is able to achieve excellent utility at around $\epsilon = 5\sim 10$, where $\epsilon$ is the standard  parameter in the language of $(\epsilon, \delta)$-DP. 

\end{abstract}

\newpage

\section{Introduction}

Let $\bm u, \bm v\in\{0,1\}^D$ be two $D$-dimensional binary vectors. In this paper, we first focus on  hashing algorithms for the binary Jaccard similarity (a.k.a. the ``resemblance'') defined as 
\begin{equation} \label{def:jaccard}
    J(\bm u,\bm v)=\frac{\sum_{i=1}^D \mathbbm 1\{u_i=v_i=1\}}{\sum_{i=1}^D \mathbbm 1\{u_i+v_i\geq 1\}}.
\end{equation}
This is a widely used similarity metric in machine learning and web applications. For example, the vectors $\bm u$ and $\bm v$ can be viewed as two sets of items represented by the locations of non-zero entries. In industrial applications with massive data size, directly calculating the pairwise Jaccard similarity among the data points becomes too costly as the sample size grows. To accelerate large-scale search and learning, the celebrated ``minwise hashing'' (MinHash) algorithm~\citep{broder1997resemblance,broder1997syntactic,broder1998min,li2005using,li2010b} is a standard hashing technique for approximating the Jaccard similarity in massive binary datasets. It has seen numerous applications such as near neighbor search, duplicate detection, malware detection, clustering, large-scale learning, social networks, and computer vision~\citep{charikar2002similarity,fetterly2003large,das2007google,buehrer2008scalable,bendersky2009finding,chierichetti2009compressing,pandey2009nearest,Proc:Lee_ECCV10,deng2012efficient,chum2012fast,he2013k,tamersoy2014guilt,shrivastava2014defense,zhu2017interactive,nargesian2018table,wang2019memory,lemiesz2021algebra,feng2021allign,jia2021bidirectionally}. 

The outputs of MinHash are integers. For large-scale applications, to store and use the hash values more conveniently and efficiently, \cite{li2010b} proposed $b$-bit MinHash that only stores the lowest $b$ bits of the hashed integers, which is memory-efficient and convenient for similarity search and machine learning. In recent years, $b$-bit MinHash has become a standard implementation of MinHash~\citep{li2011hashing,li2015using,shah2018on,yu2020hyperminhash}. In addition, we would like to mention ``circulant MinHash'' (C-MinHash)~\citep{li2022c} which, by using only one permutation circularly, improved the hashing efficiency and, perhaps surprisingly, also improved the estimation accuracy. However, the mechanism behind C-MinHash is different from (or even orthogonal to) what is commonly known as the ``one permutation hashing'' (OPH).

\subsection{One Permutation Hashing (OPH)}

To use MinHash in practice, we need to generate $K$ hash values to achieve good utility. This requires applying $K$ random permutations (or hash functions as approximations) per data point, yielding an $O(Kf)$ complexity where $f$ is the average number of non-empty entries of the data vectors. The method of one permutation hashing (OPH)~\citep{li2012one} provides a  way to significantly reduce the complexity to $O(f)$. The idea of OPH is: to generate $K$ hashes, we split the data vector into $K$ non-overlapping bins, and conduct MinHash within each bin. However, since the data are sparse, some bins might be empty. To deal with empty bins, various densification schemes~\citep{shrivastava2014densifying,shrivastava2014improved, shrivastava2017optimal,li2019re,jia2021bidirectionally} have been proposed that fill the empty bins (e.g., with values from  non-empty bins). It is shown that OPH with densification also provides unbiased Jaccard similarity estimation, and interestingly the estimation variance of certain densification algorithms might be even smaller than that of the standard MinHash. Owing to its much improved efficiency, OPH has been widely used as an alternative in practice for replacing MinHash~\citep{dahlgaard2017practical,zhao2020distributed,tseng2021parallel,jiang2022massive}. We also note that, OPH can be further improved by adopting the idea of ``circulation'', to become the ``C-OPH'' algorithm~\citep{li2021c-oph}.

\subsection{Non-binary Data and Bin-wise Consistent Weighted Sampling (BCWS)}

MinHash and OPH  handle binary data (unweighted sets). In the literature, the weighted Jaccard similarity is a popular extension to the real-valued data (weighted sets), defined as
\begin{align*}
    J_w(\bm u, \bm v)=\frac{\sum_{i=1}^D\min\{u_i,v_i\}}{\sum_{i=1}^D\max\{u_i,v_i\}},
\end{align*}
where $\bm u,\bm v\in\mathbb R_+^D$ be two non-negative data vectors. Since weighted data usually contains richer information than binary data, the weighted Jaccard  has also been studied and used in many areas including theory, databases, machine learning, and information retrieval~\citep{kleinberg1999approximation,charikar2002similarity,fetterly2003large,gollapudi2009axiomatic,bollegala11a2011web,delgado2014data,schubert2014signitrend,wang2014learning,fu2015design,pewny2015cross,manzoor2016fast,raff2017alternative,tymoshenko2018cross,bag2019efficient,pouget2019variance,yang2019nodesketch,zhu2019josie,fuchs2020intent,lei2020locality,li2022p,zheng2023building}. Particularly, the weighted Jaccard similarity has been shown to be a powerful non-linear kernel in large-scale classification and regression~\citep{li2017linearized}. It may outperform the popular RBF kernel on many tasks/datasets. 

Correspondingly, large-scale hashing algorithms have also been developed for the efficient estimation/approximation of the weighted Jaccard similarity. A series of works~\citep{kleinberg1999approximation, charikar2002similarity,shrivastava2016simple,ertl2018bagminhash,li2021rejection} proposed and refined the hashing algorithm using the rejection sampling technique which is efficient on dense data vectors. \citet{gollapudi2006exploiting,manasse2010consistent,ioffe2010improved} developed Consistent Weighted Sampling (CWS, Algorithm~\ref{alg:CWS}), which has $O(Kf)$ complexity same as that of MinHash and runs efficiently on relatively sparse data. \citet{li2021consistent} proposed Extremal Sampling (ES) built upon the extremal stochastic process. \citet{li2019re} extended the ``binning + densification'' idea of OPH to CWS, and proposed Bin-wise Consistent Weighted Sampling (BCWS) with $O(f)$ complexity, i.e., roughly $K$-fold speedup compared to standard CWS.

\subsection{Hashing/Sketching and Differential Privacy}

At a higher level, MinHash, OPH and BCWS all belong to the broad family of hashing/sketching methods which generate sketches for data samples that are designed for various applications. Examples of  sketching methods include the random projection (RP) based methods for estimating the $l_2$ distance, the inner product/cosine and their non-linear transformations~\citep{goemans1995improved, charikar2002similarity,datar2004locality, vempala2005random,rahimi2007random,li2014coding,li2017linearized,li2022signrff}, the count-sketch (CS) for frequency estimation~\citep{charikar2004finding}, and the Flajolet-Martin (FM) sketch~\citep{flajoletM1985probabilistic} and HyperLogLog sketch~\citep{flajolet2007hyperloglog} for cardinality estimation, etc. Since the data sketches produce ``summaries'' of the data which contain the original data information, protecting the privacy of the data sketches becomes an important topic that has gained growing research interest in recent years. 

The differential privacy (DP)~\citep{dwork2006calibrating} has become a standard privacy definition with rigorous mathematical formulation, which is widely applied to clustering, regression and classification, principle component analysis, matrix completion, optimization and deep learning~\citep{blum2005practical,chaudhuri2008privacy,feldman2009private,gupta2010differentially,chaudhuri2011differentially,kasiviswanathan2012analyzing,zhang2012functional,abadi2016deep,agarwal2018cpsgd,ge2018minimax,wei2020federated,dong2022gaussian,fan2022private}, etc. Moreover, prior efforts have also been conducted to combine differential privacy with the aforementioned hashing algorithms, including RP~\citep{blocki2012johnson,kenthapadi2013privacy,stausholm2021improved,li2023differential}, count-sketch~\citep{zhao2022differentially}, and FM sketch~\citep{smith2020flajolet,dickens2022order}. Recently, \citet{aumuller2020differentially} proposed DP  for MinHash. Although their design does not satisfy DP,  we will make an important (and technically minor) fix for their method. The revised method is named DP-MH and will be compared with DP-OPH, our proposed DP algorithms for OPH.

Among those DP algorithms for sketching/hashing, some works (e.g., \citet{blocki2012johnson,smith2020flajolet,dickens2022order}) assumed ``internal randomness'', i.e., the randomness of the hash functions are kept private, and showed that some hashing methods themselves already possess strong DP property, when there is a sufficient number of elements in the data vector (e.g., for the FM sketch) or the data matrix has sufficiently large eigenvalues. However, this setting is more restrictive in practice since it requires that the hash keys or projection matrices cannot be accessed by any adversary. In another setup (e.g.,~\citet{kenthapadi2013privacy,aumuller2020differentially,stausholm2021improved,zhao2022differentially,li2023differential}), both the randomness of the hash functions and the algorithm outputs are treated as public information, and perturbation mechanisms are developed to make the corresponding hashing algorithms differentially private.

\subsection{Our Contributions}

In this paper, we mainly focus on the differential privacy of one permutation hashing (OPH), for hashing the binary Jaccard similarity. We consider the  practical and general setup where the randomness (e.g., the permutation) of the algorithm is ``external'' and public. We develop three variants under the DP-OPH framework, DP-OPH-fix, DP-OPH-re, and DP-OPH-rand, corresponding to fixed densification, re-randomized densification, and random-bit densification for OPH, respectively. We provide detailed algorithm design and privacy analysis for each variant, and compare them with the DP MinHash (DP-MH) method. In our retrieval experiments, we show that the proposed DP-OPH method substantially improves DP-MH, and re-randomized densification (DP-OPH-re) is superior to fixed densification (DP-OPH-fix) in terms of differential privacy. The DP-OPH-rand variant performs the best when $\epsilon$ is small, while DP-OPH-re is the strongest competitor in the larger $\epsilon$ region. Here, $\epsilon$ is the standard parameter in the $(\epsilon,\delta)$-DP language. Moreover, we also empirically evaluate an extension of our method to the weighted Jaccard similarity called DP-BCWS for non-binary data. Experiments on classification demonstrate the effectiveness of the proposed methods in machine learning tasks under privacy constraints.

%\vspace{0.1in}

% In the literature, the only prior work on private MinHash that we are aware of is \cite{aumuller2020differentially}, which proposed a bit flipping technique for $b$-bit MinHash based on the generalized randomized response technique. However, we point out a minor error in the construction of \cite{aumuller2020differentially} where a tail bound was mistakenly used, which led to less perturbation than what is really required by differential privacy. We make a correction in our DP-MinHash (DP-MH) algorithm. 

\section{Background on MinHash, $b$-bit Coding, and Differential Privacy}

\begin{algorithm}[h]
	\textbf{Input:} Binary vector $\bm u\in\{0,1\}^D$; number of hash values $K$

	\textbf{Output:} $K$ MinHash values $h_1(\bm u),...,h_K(\bm u)$
	
\begin{algorithmic}[1]	

	\State Generate $K$ independent permutations $\pi_1,...,\pi_K$: $[D]\rightarrow[D]$ with seeds $1,...,K$ respectively
	
	\For{ $k=1$ to $K$}

	\State $h_k(\bm u) \leftarrow \min_{i:u_i\neq 0} \pi_k(i)$

	\EndFor
\end{algorithmic}
\caption{Minwise hashing (MinHash)}
\label{alg:MinHash}
\end{algorithm}

\noindent\textbf{Minwise hashing (MinHash).} The algorithm is summarized in Algorithm~\ref{alg:MinHash}. We first generate $K$ independent permutations  $\pi_1,...,\pi_K:[D]\mapsto[D]$, where the seeds ensure that all data vectors use the same set of permutations. Here, $[D]$ denotes $\{1,...,D\}$. For each permutation, the hash value is simply the first non-zero location in the permuted vector, i.e., $h_k(\bm u)=\min_{i: v_i\neq 0} \pi_k(i)$, $\forall k=1,...,K$. Analogously, for another data vector $\bm v\in\{0,1\}^D$, we also obtain $K$ hash values, $h_k(\bm v)$.

\newpage

\noindent The MinHash estimator of $J(\bm u, \bm v)$ is the average over the hash collisions:
\begin{align}\label{MH-estimator}
    \hat J_{MH}(\bm u,\bm v)=\frac{1}{K}\sum_{k=1}^K \mathbbm 1\{h_k(\bm u)=h_k(\bm v)\},
\end{align}
where $\mathbbm 1\{\cdot\}$ is the indicator function. By some standard probability calculations, we can show that
\begin{align*}
    \mathbb E[\hat J_{MH}]=J,\hspace{0.4in} Var[\hat J_{MH}]=\frac{J(1-J)}{K}.
\end{align*}

\vspace{0.2in}
\noindent\textbf{$b$-bit coding of the hash value.}  \cite{li2010b} proposed ``$b$-bit minwise hashing'' as a convenient coding strategy for the integer hash value $h(\bm u)$ generated by MinHash (or by OPH which will be introduced later). Basically, we only keep the lowest $b$ bits of each hash value. In our analysis, for convenience, we assume that ``taking the lowest $b$-bits'' can be achieved by some ``rehashing'' trick to map the integer values onto $\{0,...,2^b-1\}$ uniformly. There are at least three benefits of this coding strategy: (i) storing only $b$ bits saves the storage cost compared with storing the full 32 or 64 bit values; (ii) the lowest few bits are more convenient for the purpose of indexing, e.g., in approximate nearest neighbor search~\citep{indyk1998approximate,shrivastava2012fast,shrivastava2014defense}; (iii) we can transform the lowest few bits into a positional representation, allowing us to approximate the Jaccard similarity by the inner product of the hashed data, which is required by training large-scale linear models~\citep{li2011hashing}. In this work, we will adopt this $b$-bit coding strategy in our private algorithm design.

\vspace{0.2in}
\noindent\textbf{Differential privacy (DP).} Differential privacy (DP) provides a rigorous mathematical definition on the privacy of a randomized algorithm. Its formal definition is as follows.

\begin{definition}[Differential privacy~\citep{dwork2006calibrating}] \label{def:DP}
For a randomized algorithm $\mathcal M:\mathcal U\mapsto Range(\mathcal M)$ and $\epsilon,\delta\geq 0$, if for any two neighboring data $u$ and $u'$, it holds that
\begin{equation} \label{eq:DP-def}
    Pr[\mathcal M(u)\in Z] \leq e^\epsilon Pr[\mathcal M(u')\in Z]+\delta
\end{equation}
for $\forall Z\subset Range(\mathcal M)$,
then algorithm $\mathcal M$ is said to satisfy $(\epsilon,\delta)$-differentially privacy. If $\delta=0$, $\mathcal M$ is called $\epsilon$-differentially private.
\end{definition}

Intuitively, DP requires that the distributions of the outputs before and after a small change in the data are similar such that an adversary cannot detect the change based on the outputs with high confidence. Smaller $\epsilon$ and $\delta$ implies stronger privacy. The parameter $\delta$ is usually interpreted as the ``failure probability'' allowed for the $\epsilon$-DP guarantee to be violated. In our work, we follow the standard definition in the DP literature of ``neighboring'' for binary data vectors, e.g.,~\citet{dwork2006calibrating,kenthapadi2013privacy,xu2013differentially,dwork2014algorithmic,smith2020flajolet,stausholm2021improved,dickens2022order,zhao2022differentially,li2023differential}, that two data vectors are adjacent if they differ in one element. The formal statement is as below.

\vspace{0.1in}

\begin{definition}[Neighboring data] \label{def:neighbor}
$\bm u,\bm u'\in\{0,1\}^D$ are neighboring data vectors if they only differ in one element, i.e., $u_i\neq u_i'$ for one $i\in[D]$ and $u_j=u_j'$ for $j\neq i$.
\end{definition}

\newpage

\section{Hashing for Jaccard Similarity with Differential Privacy (DP)}

As discussed earlier, one permutation hashing (OPH)~\citep{li2012one} is a popular and highly efficient hashing algorithm for the Jaccard similarity. In this section, we present our main algorithms called DP-OPH based on privatizing the $b$-bit hash values from OPH. In addition, we compare it with a differentially private MinHash alternative named DP-MH.

\subsection{One Permutation Hashing (OPH)}
\label{sec:DP-OPH}

\begin{algorithm}[h]
	\textbf{Input:} Binary vector $\bm u\in\{0,1\}^D$; number of hash values $K$

	\textbf{Output:} OPH hash values $h_1(\bm u),...,h_K(\bm u)$
	
\begin{algorithmic}[1]	
	\State Let $d=D/K$. Use a permutation $\pi:[D]\mapsto [D]$ with fixed seed to randomly split $[D]$ into $K$ equal-size bins $\mathcal B_1,...,\mathcal B_K$, with $\mathcal B_k=\{j\in [D]:(k-1)d+1\leq \pi(j)\leq kd\}$
	
	\For{ $k=1$ to $K$}
	
	\If{Bin $\mathcal B_k$ is non-empty}
	
	\State $h_k(\bm u)\leftarrow \min_{j\in \mathcal B_k, u_j\neq 0} \pi(j)$
	
	\Else
	
	\State $h_k(\bm u)\leftarrow E$
	
	\EndIf

	\EndFor
\end{algorithmic}
\caption{One Permutation Hashing (OPH)}
\label{alg:OPH}
\end{algorithm}

As outlined in Algorithm~\ref{alg:OPH}, the procedure of OPH is simple: we first use a permutation $\pi$ (same for all data vectors) to randomly split the feature dimensions $[D]$ into $K$ bins $\mathcal B_1,...,\mathcal B_K$ with equal length $d=D/K$ (assuming integer division holds). Then, for each bin $\mathcal B_k$, we set the minimal permuted index of ``1'' as the $k$-th OPH hash value. If $\mathcal B_k$ is empty (i.e., it does not contain any ``1''), we record an ``$E$'' representing an empty bin. \cite{li2012one} showed that we can construct statistically unbiased Jaccard estimators by ignoring the empty bins (i.e., subtracting $K$ by the number of ``jointly'' empty bins). However, since empty bins are different for every distinct data vector, the vanilla OPH hash values do not form a metric space (i.e., do not satisfy the triangle inequality). This limits the application of this naive estimator.

\begin{algorithm}[tb]
	\textbf{Input:} OPH hash values $h_1(\bm u),...,h_K(\bm u)$ each in $[D]\cup \{E\}$; bins $\mathcal B_1,...,\mathcal B_K$; $d=D/K$

	\textbf{Output:} Densified OPH hash values $h_1(\bm u),...,h_K(\bm u)$
	
\begin{algorithmic}[1]	
    \State Let $NonEmptyBin=\{k\in [K]:h_k(\bm u)\neq E\}$

	\For{$k=1$ to $K$}
	
	\If{$h_k(\bm u)=E$}
	
	\State Uniformly randomly select $k'\in NonEmptyBin$
	
	{\color{blue}\State  $h_k(\bm u)\leftarrow h_{k'}(\bm u)$  \Comment{OPH-fix: fixed densification}}
	
	\State \textbf{Or}
	
	{\color{red} \State  $MapToIndex=SortedIndex \left(\pi(\mathcal B_k) \right)+(k'-1)d$
	
	\State $\pi^{(k)}: \pi(\mathcal B_{k'})\mapsto MapToIndex$  \Comment{within-bin partial permutation}
	
	\State $h_k(\bm u)\leftarrow \min_{j\in \mathcal B_{k'}, u_j\neq 0} \pi^{(k)}\left( \pi(j) \right) $ \Comment{OPH-re: re-randomized densification}}

	\EndIf
	
	\EndFor
	
\end{algorithmic}
\caption{Densification for OPH, two options: {\color{blue} fixed} and {\color{red} re-randomized}  }
\label{alg:densification}
\end{algorithm}

\vspace{0.2in}
\noindent\textbf{Densification for OPH.} To tackle the issue caused by empty bins, a series of works have been conducted to densify the OPH~\citep{shrivastava2014densifying,shrivastava2014improved,shrivastava2016simple,li2019re,jia2021bidirectionally}. The general idea is to ``borrow'' the data/hash from non-empty bins, with some careful design. In Algorithm~\ref{alg:densification}, we present two recent representatives of OPH densification methods: ``{\color{blue}fixed densification}''~\citep{shrivastava2017optimal} and ``{\color{red} re-randomized densification}''~\citep{li2019re}. We call these two variants OPH-fix and OPH-re, respectively. Given an OPH hash vector from Algorithm~\ref{alg:OPH} (possibly containing ``$E$''s), we denote the set of non-empty bins $NonEmptyBin=\{k\in [K]:h_k(\bm u)\neq E\}$. The densification procedure scans over $k=1,...,K$. For each $k$ with $h_k(\bm u)=E$, we do the following. 

\begin{enumerate}
    \item Uniformly randomly pick a bin $k'\in NonEmptyBin$ that is non-empty.
    
    \item
    \begin{enumerate}
        \item {\color{blue} OPH-fix:} we directly copy the $k'$-th hash value: $h_k(\bm u)\leftarrow h_{k'}(\bm u)$.
        
        \item {\color{red} OPH-re:} we apply an additional minwise hashing to bin $\mathcal B_{k'}$ using the ``partial permutation'' of $\mathcal B_k$ to obtain the hash for $h_k(\bm u)$. 
    \end{enumerate}
    
\end{enumerate}

\newpage

More precisely, for re-randomized densification, in Algorithm~\ref{alg:densification}, $MapToIndex$ defines the ``partial permutation'' of bin $\mathcal B_k$, where the function $SortedIndex$ returns the original index of a sorted array. For example, let $D=16$, $K=4$, and $d=D/K=4$ and suppose the indices in each bin are in ascending order, and $\mathcal B_2=[1,5,13,15]$ is empty. Suppose $\pi(13)=5, \pi(5)=6, \pi(1)=7, \pi(15)=8$. In this case, $\pi(\mathcal B_2)=[7,6,5,8]$, so $SortedIndex(\pi(\mathcal B_2))=[3,2,1,4]$. Assume $k'=3$ is picked and $\pi(\mathcal B_3)=[9,12,10,11]$. At line 7 we have $MapToIndex=[11,10,9,12]$ and at line 8, $\pi^{(2)}$ is a mapping $[9,12,10,11]\mapsto [11,10,9,12]$, which effectively defines another within-bin permutation of $\pi(\mathcal B_3)$ using the partial ordering of $\pi(\mathcal B_2)$. Finally, we set $h_k(\bm u)$ as the minimal index of ``1'' among the additionally permuted elements in bin $\mathcal B_{k'}$.

\vspace{0.1in}

We remark that in the above step 1, for any empty bin $k$, the ``sequence'' for non-empty bin lookup should be the same for any data vector. In practice, this can be achieved by simply seeding a random permutation of $[K]$ for each $k$. For instance, for $k=1$ (when the first bin is empty), we always search in the order $[3,1,2,4]$ until one non-empty bin is found, for all the data vectors.

It is shown that for both variants, the Jaccard estimator of the same form as (\ref{MH-estimator}) is unbiased. In \cite{li2019re}, the authors proved that re-randomized densification always achieves a smaller Jaccard estimation variance than that of fixed densification, and the improvement is especially significant when the data is sparse. Similar to $b$-bit minwise hashing, we can also keep the lowest $b$ bits of the OPH hash values to use them conveniently in search and learning applications. 

\subsection{Differential Private One Permutation Hashing (DP-OPH)}

\begin{algorithm}[h]
	\textbf{Input:} Densified OPH hash values $h_1(\bm u),...,h_K(\bm u)$; number of bits $b$; $\epsilon>0$, $0<\delta<1$
	
	\hspace{0.38in} $f$: lower bound on the number of non-zeros in each data vector

	\textbf{Output:} $b$-bit DP-OPH values $\tilde h(\bm u)=[\tilde h_1(\bm u),...,\tilde h_K(\bm u)]$
	
\begin{algorithmic}[1]	

    \State Take the lowest $b$ bits of all hash values \Comment{After which $h_k(\bm u)\in \{0,...,2^b-1\}$}
    
    \State Set {\color{blue} $N=F_{fix}^{-1}(1-\delta;D,K,f)$ (for DP-OPH-fix)} or {\color{red} $N=F_{re}^{-1}(1-\delta;D,K,f)$ (for DP-OPH-re)}, and $\epsilon'=\epsilon/N$

	\For{ $k=1$ to $K$}
	
	\State $\tilde h_k(\bm u)=
    \begin{cases}
    h_k(\bm u), & \text{with probability}\ \frac{e^{\epsilon'}}{e^{\epsilon'}+2^b-1}\\
    i, & \text{with probability}\ \frac{1}{e^{\epsilon'}+2^b-1},\ \text{for}\ i\in \{0,...,2^b-1\},\ i\neq h_k(\bm u)
    \end{cases}$

	\EndFor
\end{algorithmic}
\caption{Differentially Private Densified One Permutation Hashing (DP-OPH-fix, DP-OPH-re)}
\label{alg:DP-OPH-densification}
\end{algorithm}

\noindent\textbf{DP-OPH with densification.} To privatize densified OPH, in Algorithm~\ref{alg:DP-OPH-densification}, we first take the lowest $b$ bits of the hash values. Since the output space is finite with cardinality $2^b$, we apply the standard randomized response technique~\citep{warner1965randomized,dwork2014algorithmic,wang2017locally} to flip the bits to achieve DP. After running Algorithm~\ref{alg:densification}, suppose a densified OPH hash value $h_k(\bm u)=j$ with some $j\in 0,...,2^b-1$. With some $\epsilon'>0$ that will be specified later, we output $\tilde h_k(\bm u)=j$ with probability $\frac{e^{\epsilon'}}{e^{\epsilon'}+2^b-1}$, and $\tilde h_k(\bm u)=i$ for $i\neq j$ with probability $\frac{1}{e^{\epsilon'}+2^b-1}$. It is easy to verify that, for a neighboring data $\bm u'$, when $h_{k}(\bm u')=j$, for $\forall i\in 0,...,2^b-1$, we have $\frac{P(\tilde h_k(\bm u)=i)}{P(\tilde h_k(\bm u')=i)}=1$; when $h_{k}(\bm u')\neq j$, we have $e^{-\epsilon'}\leq \frac{P(\tilde h_k(\bm u)=i)}{P(\tilde h_k(\bm u')=i)}\leq e^{\epsilon'}$. Therefore, for a single hash value, this bit flipping method satisfies $\epsilon'$-DP.

\vspace{0.1in}

It remains to determine $\epsilon'$. Naively, since the perturbations (flipping) of the hash values are independent, by the composition property of DP~\citep{dwork2006our}, we know that simply setting $\epsilon'=\epsilon/K$ for all $K$ MinHash values would achieve overall $\epsilon$-DP (for the hashed vector). However, since $K$ is usually around hundreds, a very large $\epsilon$ value is required for this strategy to be useful. To this end, we can trade a small $\delta$ in the DP definition for a significantly reduced $\epsilon$. The key is to note that, not all the $K$ hashed bits will change after we switch from $\bm u$ to its neighbor $\bm u'$. Assume each data vector contains at least $f$ non-zeros, which is realistic since many data in practice have both high dimensionality $D$ as well as many non-zero elements. Intuitively, when the data is not too sparse, $\bm u$ and $\bm u'$ tend to be similar (since they only differ in one element). Thus, the number of different hash values from Algorithm~\ref{alg:densification}, $X=\sum_{k=1}^K \mathbbm 1\{h_k(\bm u)\neq h_k(\bm u')\}$, can be upper bounded by some $N$ with high probability $1-\delta$. In the proof, this allows us to set $\epsilon'=\epsilon/N$ in the flipping probabilities and count $\delta$ as the failure probability in $(\epsilon,\delta)$-DP. In Lemma~\ref{lemma:distribution-of-X}, we derive the precise probability distribution of $X$. Based on this result, in Algorithm~\ref{alg:DP-OPH-densification}, we set $N=F_{fix}^{-1}(1-\delta;D,f,K)$ for DP-OPH-fix, $N=F_{re}^{-1}(1-\delta;D,f,K)$ for DP-OPH-re, where $F_{fix}(x)=P(X\leq x)$ is the cumulative mass function (CMF) of $X$ with OPH-fix ((\ref{eqn:distribution-of-X}) + (\ref{eqn:Prob-fix})), and $F_{re}$ is the cumulative mass function of $X$ with OPH-re ((\ref{eqn:distribution-of-X}) + (\ref{eqn:Prob-re})). To prove Lemma~\ref{lemma:distribution-of-X}, we need the following two lemmas.

\begin{lemma}[\cite{li2012one}]  \label{lemma:number of empty}
Let $f=|\{i:u_i=1\}|$, and $I_{emp,k}$ be the indicator function that the $k$-th bin is empty, and $N_{emp}=\sum_{k=1}^K I_{emp,k}$. Suppose $mod(D,K)=0$. We have
\begin{align*}
	P\left( N_{emp} = j\right)&= \sum_{\ell=0}^{K-j}(-1)^\ell {K\choose j}{K-j\choose \ell}
	{D(1-(j+\ell)/K)\choose f}\bigg/{D\choose f}. 
	\end{align*}
\end{lemma}

\begin{lemma}[\cite{li2019re}]  \label{lemma:bin-non-zero}
Conditional on the event that $m$ bins are non-empty, let $\tilde f$ be the number of non-zero elements in a non-empty bin. Denote $d=D/K$. The conditional probability distribution of $\tilde f$ is given by
\begin{align*}
	P\left( \tilde f = j\big| m\right)&= \frac{{d \choose j} H(m-1,f-j|d)}{H(m,f|d)},\quad j=\max\{1,f-(m-1)d\},...,\min\{d,f-m+1\},
\end{align*}
where $H(\cdot)$ follows the recursion: for any $0< k\leq K$ and $0\leq n\leq f$,
\begin{align*}
    H(k,n|d)=\sum_{i=\max\{1,n-(k-1)d\}}^{\min\{d,n-k+1\}} {d \choose i} H(k-1,n-i|d),\quad H(1,n|d)={d \choose n}.
\end{align*}
\end{lemma}

The distribution of $X=\sum_{k=1}^K \mathbbm 1\{h_k(\bm u)\neq h_k(\bm u')\}$ is given as below. 

\begin{lemma}  \label{lemma:distribution-of-X}
Consider $\bm u\in \{0,1\}^D$, and denote $f=|\{i:u_i=1\}|$. Let $\bm u'$ be a neighbor of $\bm u$. Denote $X=\sum_{k=1}^K \mathbbm 1\{h_k(\bm u)\neq h_k(\bm u')\}$ where the hash values are generated by Algorithm~\ref{alg:densification}. Denote $d=D/K$. We have, for $x=0,...,K-\lceil f/d \rceil$,
\begin{align}
    P\left( X=x \right)=\sum_{j=\max(0,K-f)}^{K-\lceil f/d \rceil}\sum_{z=1}^{\min(f,d)} \tilde P(x|z,j) P\left(\tilde f=z | K-j\right)P\left( N_{emp} = j\right),  \label{eqn:distribution-of-X}
\end{align}
where $P\left(\tilde f=z | K-j\right)$ is given in Lemma~\ref{lemma:bin-non-zero} and $P\left( N_{emp} = j\right)$ is from Lemma~\ref{lemma:number of empty}. Moreover,
\begin{align}
    \textbf{For OPH-fix:}\ \ &\tilde P(x|z,j)=\mathbbm 1\{x=0\} \left(1-P_{\neq}\right) + \mathbbm 1\{x>0\}P_{\neq}\cdot g_{bino}\left(x-1;\frac{1}{K-j},j\right), \label{eqn:Prob-fix} \\
    \textbf{For OPH-re:}\ \ &\tilde P(x|z,j)=\left(1-P_{\neq}\right) \cdot g_{bino}\left(x;\frac{P_{\neq}}{K-j},j\right) + P_{\neq} \cdot g_{bino}\left(x-1;\frac{P_{\neq}}{K-j},j\right), \label{eqn:Prob-re}
\end{align}
where $g_{bino}(x;p,n)$ is the probability mass function of $Binomial(p,n)$ with $n$ trials and success rate $p$, and $P_{\neq}(z,b)= \left(1-\frac{1}{2^b}\right)\frac{1}{z}$.
\end{lemma}

\begin{proof}
Without loss of generality, suppose $\bm u$ and $\bm u'$ differ in the $i$-th dimension, and by the symmetry of DP, we can assume that $u_i=1$ and $u_i'=0$. We know that $i$ is assigned to the $\lceil mod(\pi(i),d)\rceil$-th bin. Among the $K$ hash values, this change will affect all the bins that use the data/hash of the $k^*=\lceil mod(\pi(i),d)\rceil$-th bin (after permutation), both in the first scan (if it is non-empty) and in the densification process. Let $N_{emp}$ be the number of empty bins in $h(\bm u)$, and $\tilde f$ be the number of non-zero elements in the $k^*$-th bin. We have, for $x=0,...,K-\lceil f/d \rceil$,
\begin{align*}
    P\left( X=x \right) 
    &=\sum_{j=\max(0,K-f)}^{K-\lceil f/d \rceil}\sum_{z=1}^{\min(f,d)} P\left( X=x \Big|\tilde f=z, N_{emp}=j\right) P\left(\tilde f=z,N_{emp} = j\right) \\
    &=\sum_{j=\max(0,K-f)}^{K-\lceil f/d \rceil}\sum_{z=1}^{\min(f,d)} P\left( X=x \Big|\tilde f=z, N_{emp}=j\right) P\left(\tilde f=z | K-j\right)P\left( N_{emp} = j\right),
\end{align*}
where $P\left(\tilde f=z | K-j\right)$ is given in Lemma~\ref{lemma:bin-non-zero} and $P\left( N_{emp} = j\right)$ can be calculated by Lemma~\ref{lemma:number of empty}. To compute the first conditional probability, we need to compute the number of times the $k^*$-th bin is picked to generate hash values, and the hash values are different for $\bm u$ and $\bm u'$. Conditional on $\{\tilde f=z, N_{emp}=j\}$, denote $\Omega=\{k:\mathcal B_k\ \text{is\ empty}\}$, and let $R_k$ be the non-empty bin used for the $k$-th hash value $h_k(\bm u)$, which takes value in $[K]\setminus \Omega$. We know that $|\Omega|=j$. We can write
\begin{align*}
    X=\mathbbm 1\{h_{k^*}(\bm u)\neq h_{k^*}(\bm u')\}+\sum_{k\in \Omega}\mathbbm 1\{R_k=k^*, h_k(\bm u)\neq h_k(\bm u')\}.
\end{align*}
Here we separate out the first term because the $k^*$-th hash always uses the $k^*$-bin. Note that the densification bin selection is uniform, and the bin selection is independent of the permutation for hashing. For the fixed densification, since the hash value $h_{k^*}(\bm u)$ is generated and used for all hash values that use $\mathcal B_{k^*}$, we have
\begin{align*}
    P\left( X=x \Big|\tilde f=z, N_{emp}j\right)=\mathbbm 1\{x=0\} \left(1-P_{\neq}\right) + \mathbbm 1\{x>0\}P_{\neq}\cdot g_{bino}\left(x-1;\frac{1}{K-j},j\right),
\end{align*}
where $g_{bino}(x;p,n)$ is the probability mass function of the binomial distribution with $n$ trials and success rate $p$, and $P_{\neq}= P(h_{k^*}(\bm u)\neq h_{k^*}(\bm u'))= \left(1-\frac{1}{2^b}\right)\frac{1}{z}$. Based on the same reasoning, for re-randomized densification, we have
\begin{align*}
    P\left( X=x \Big|\tilde f=z, N_{emp}j\right)=\left(1-P_{\neq}\right) \cdot g_{bino}\left(x;\frac{P_{\neq}}{K-j},j\right) + P_{\neq} \cdot g_{bino}\left(x-1;\frac{P_{\neq}}{K-j},j\right).
\end{align*}
Combining all the parts completes the proof.
\end{proof}

The privacy guarantee of DP-OPH with densification is provided as below.

\begin{theorem} \label{theo:DP-OPH-densification}
Both DP-OPH-fix and DP-OPH-re in Algorithm~\ref{alg:DP-OPH-densification} achieve $(\epsilon,\delta)$-DP.
\end{theorem}
\begin{proof}
Let $\bm u$ and $\bm u'$ be neighbors only differing in one element. Denote $S=\{k\in [K]:h_k(\bm u)\neq h_k(\bm u')\}$ and $S^c=[K]\setminus S$. As discussed before, we can verify that for $k\in S_c$, we have $\frac{P(\tilde h_k(\bm u)=i)}{P(\tilde h_k(\bm u')=i)}=1$ for any $i=0,...,2^b-1$. For $k\in S$, $e^{-\epsilon'}\leq \frac{P(\tilde h_k(\bm u)=i)}{P(\tilde h_k(\bm u')=i)}\leq e^{\epsilon'}$ holds for any $i=0,...,2^b-1$. Thus, for any $Z\in \{0,...,2^b-1\}^K$, the absolute privacy loss can be bounded by
\begin{align}
    \left|\log\frac{P(\tilde h(\bm u)=Z)}{P(\tilde h(\bm u')=Z)}\right|=\left|\log\prod_{k\in S} \frac{P(\tilde h_k(\bm u)=i)}{P(\tilde h_k(\bm u')=i)}\right|\leq |S|\epsilon'= |S|\frac{\epsilon}{N}.  \label{eqn:DP-OPH-1}
\end{align}
By Lemma~\ref{lemma:distribution-of-X}, we know that with probability at least $1-\delta$, $|S|\leq F_{fix}^{-1}(1-\delta)=N$ for DP-OPH-fix, and $|S|\leq F_{re}^{-1}(1-\delta)=N$ for DP-OPH-re, respectively. Hence, (\ref{eqn:DP-OPH-1}) is upper bounded by $\epsilon$ with probability $1-\delta$. This proves the $(\epsilon,\delta)$-DP.
\end{proof}

\noindent\textbf{DP-OPH without densification.} From the practical perspective, we may also privatize the OPH without densification (i.e., add DP to the output of Algorithm~\ref{alg:OPH}). The first step is to take the lowest $b$ bits of every non-empty hash and get $K$ hash values from $\{0,...,2^b-1\}\cup \{E\}$. Then, for non-empty bins, we keep the hash value with probability $\frac{e^{\epsilon}}{e^{\epsilon}+2^b-1}$, and randomly flip it otherwise. For empty bins (i.e., $h_k(\bm u)=E$), we simply assign a random value in $\{0,...,2^b-1\}$ to $\tilde h_k(\bm u)$. The formal procedure of this so-called DP-OPH-rand method is summarized in Algorithm~\ref{alg:DP-OPH-randbit}.

\begin{algorithm}[h]
	\textbf{Input:} OPH hash values $h_1(\bm u),...,h_K(\bm u)$ from Algorithm~\ref{alg:OPH}; number of bits $b$; $\epsilon>0$

	\textbf{Output:} DP-OPH-rand hash values $\tilde h(\bm u)=[\tilde h_1(\bm u),...,\tilde h_K(\bm u)]$
	
\begin{algorithmic}[1]	

    \State Take the lowest $b$ bits of all hash values \Comment{After which $h_k(\bm u)\in \{0,...,2^b-1\}$}

	\For{ $k=1$ to $K$}
	
	\If{$h_k(\bm u)\neq E$}
	\State $\tilde h_k(\bm u)=
    \begin{cases}
    h_k(\bm u), & \text{with probability}\ \frac{e^{\epsilon}}{e^{\epsilon}+2^b-1}\\
    i, & \text{with probability}\ \frac{1}{e^{\epsilon'}+2^b-1},\ \text{for}\ i\in \{0,...,2^b-1\},\ i\neq h_k(\bm u)
    \end{cases}$
    
    \Else
    
    \State $\tilde h_k(\bm u)=i$ with probability $\frac{1}{2^b}$, for $i=0,...,2^b-1$  \Comment{Assign random bits to empty bin}
    
    \EndIf

	\EndFor
\end{algorithmic}
\caption{Differentially Private One Permutation Hashing with Random Bits (DP-OPH-rand)}
\label{alg:DP-OPH-randbit}
\end{algorithm}

\newpage

\begin{theorem}  \label{theo:DP-OPH-rand}
Algorithm~\ref{alg:DP-OPH-randbit} achieves $\epsilon$-DP.
\end{theorem}
\begin{proof}
The proof is similar to the proof of Theorem~\ref{theo:DP-OPH-densification}. Since the original hash vector $h(\bm u)$ is not densified, there only exists exactly one hash value such that $h_k(\bm u)\neq h_k(\bm u)$ may happen for $\bm u'$ that differs in one element from $\bm u$. W.l.o.g., assume $u_i=1$ and $u_i'=0$, and $i\in \mathcal B_k$. If bin $k$ is non-empty for both $\bm u$ and $\bm u'$ (after permutation), then for any $Z\in \{0,...,2^b-1\}^K$, $\left|\log\frac{P(\tilde h(\bm u)=Z)}{P(\tilde h(\bm u')=Z)}\right|\leq \epsilon$ according to our analysis in Theorem~\ref{theo:DP-OPH-densification} (the probability of hash in $[K]\setminus \{k\}$ cancels out). If bin $k$ is empty for $\bm u'$, since $1\leq \frac{e^{\epsilon}}{e^{\epsilon}+2^b-1}/\frac{1}{2^b}\leq e^\epsilon$ and $e^{-\epsilon}\leq \frac{1}{2^b}/\frac{1}{e^{\epsilon}+2^b-1}\leq 1$, we also have $\left|\log\frac{P(\tilde h(\bm u)=Z)}{P(\tilde h(\bm u')=Z)}\right|\leq \epsilon$. Therefore, the algorithm is $\epsilon$-DP.
\end{proof}

\vspace{0.1in}

Compared with Algorithm~\ref{alg:DP-OPH-densification}, DP-OPH-rand has smaller bit flipping probability (effectively, $N\equiv 1$ in Algorithm~\ref{alg:DP-OPH-densification}). This demonstrates the essential benefit of ``binning'' in OPH, since the change in one data coordinate will only affect one hash value (if densification is not applied). As a consequence, the non-empty hash values are less perturbed in DP-OPH-rand than in DP-OPH-fix or DP-OPH-re. However, this comes with an extra cost as we have to assign random bits to empty bins, which do not provide  useful information about the data. Moreover, this extra cost does not diminish as $\epsilon$ increases, because the number of empty bins only depends on the data  and $K$. Our experiments in Section~\ref{sec:experiment} will provide more intuition on properly balancing this trade-off in practice.

\subsection{Differential Private MinHash (DP-MH)}   \label{sec:DP-MH}

\begin{algorithm}[h]
	\textbf{Input:} MinHash values $h_1(\bm u),...,h_K(\bm u)$; number of bits $b$; $\epsilon>0$, $0<\delta<1$
	
	\hspace{0.38in} $f$: lower bound on the number of non-zeros in each data vector

	\textbf{Output:} DP-MH hash values $\tilde h(\bm u)=[\tilde h_1(\bm u),...,\tilde h_K(\bm u)]$
	
\begin{algorithmic}[1]	

    \State Take the lowest $b$ bits of all hash values \Comment{After which $h_k(\bm u)\in \{0,...,2^b-1\}$}

    \State Set {\color{red} $N=F_{bino}^{-1}(1-\delta;\frac{1}{f},K)$}, and $\epsilon'=\epsilon/N$ \Comment{This is different from~\citet{aumuller2020differentially}}

	\For{ $k=1$ to $K$}
	
	\State $\tilde h_k(\bm u)=
    \begin{cases}
    h_k(\bm u), & \text{with probability}\ \frac{e^{\epsilon'}}{e^{\epsilon'}+2^b-1}\\
    i, & \text{with probability}\ \frac{1}{e^{\epsilon'}+2^b-1},\ \text{for}\ i\in \{0,...,2^b-1\},\ i\neq h_k(\bm u)
    \end{cases}$

	\EndFor
\end{algorithmic}
\caption{Differentially Private Minwise hashing (DP-MH)}
\label{alg:DP-MinHash}
\end{algorithm}

While we have presented our main contributions on the DP-OPH algorithms, we also discuss the DP MinHash (DP-MH) method (Algorithm~\ref{alg:DP-MinHash}) as a baseline comparison. The general mechanism of DP-MH is the same as densified DP-OPH . The main difference between Algorithm~\ref{alg:DP-MinHash} and Algorithm~\ref{alg:DP-OPH-densification} is in the calculation of $N$. In Algorithm~\ref{alg:DP-MinHash}, we set $N=F_{bino}^{-1}(1-\delta;\frac{1}{f},K)$ where $F_{bino}(x;p,n)$ is the cumulative mass function of $Binomial(p,n)$ with $n$ trials and success probability $p$.

\begin{theorem}  \label{theo:DP-MH-privacy}
Algorithm~\ref{alg:DP-MinHash} is $(\epsilon,\delta)$-DP.
\end{theorem}
\begin{proof}
We use the same proof strategy for Theorem~\ref{theo:DP-OPH-densification}. To compute the high probability bound on the number of difference hashes $X=\sum_{k=1}^K \mathbbm 1\{h_k(\bm u)\neq h_k(\bm u')\}$ for two neighboring data vectors $\bm u$ and $\bm u'$, note that in MinHash, $\mathbbm 1\{h_k(\bm u)\neq h_k(\bm u')\}$ are independent Bernoulli random variables with success rate $1/f$. Thus, $X$ follows $Binomial(\frac{1}{f},K)$. With probability $1-\delta$, $X\leq F_{bino}^{-1}(1-\delta;\frac{1}{f},K)=N$. With this key quantity, the remaining part of the proof follows.
\end{proof}

\newpage

In a related work, \cite{aumuller2020differentially} also proposed to apply randomized response to MinHash. However, the authors incorrectly used a tail bound for the binomial distribution (see their Lemma 1) which is only valid for small deviation. In DP, $\delta$ is often very small (e.g., $10^{-6}$), so the large deviation tail bound should be used which is looser than the one used therein\footnote{For $X$ following a Binomial distribution with mean $\mu$, \cite{aumuller2020differentially} used the concentration inequality $P(X\geq (1+\xi)\mu)\leq \exp(-\frac{\xi^2\mu}{3})$, which only holds when $0\leq \xi\leq 1$. For large deviations (large $\xi$), the valid Binomial tail bound should be $P(X\geq (1+\xi)\mu)\leq \exp(-\frac{\xi^2\mu}{\xi+2})$.}. That said, in their paper, the perturbation is underdetermined and their method does not satisfy DP rigorously. In  Algorithm~\ref{alg:DP-MinHash}, we fix this by using the exact probability mass function to compute the tail probability, which also avoids any loss due to the use of concentration bounds.

\subsection{Comparison: Densified DP-OPH versus DP-MH}

We compare $N$, the ``privacy discount factor'', in DP-OPH-fix, DP-OPH-re (Algorithm~\ref{alg:DP-OPH-densification}) and DP-MH (Algorithm~\ref{alg:DP-MinHash}). Smaller $N$ implies a smaller bit flipping probability which benefits the utility. In Figure~\ref{fig:N}, we plot $N$ vs. $f$, for $D=1024$, $K=64$, and $\delta=10^{-6}$. A similar comparison also holds for other $D,K$ combinations. From the figure, we observe that $N$ in DP-OPH is typically smaller than that in DP-MH. Moreover, we also see that $N$ for DP-OPH-re is consistently smaller than that for DP-OPH-fix. This illustrates that re-randomization in the densification process is an important step to ensure better privacy.

\begin{figure}[h]

\mbox{\hspace{-0.15in}
    \includegraphics[width=2.3in]{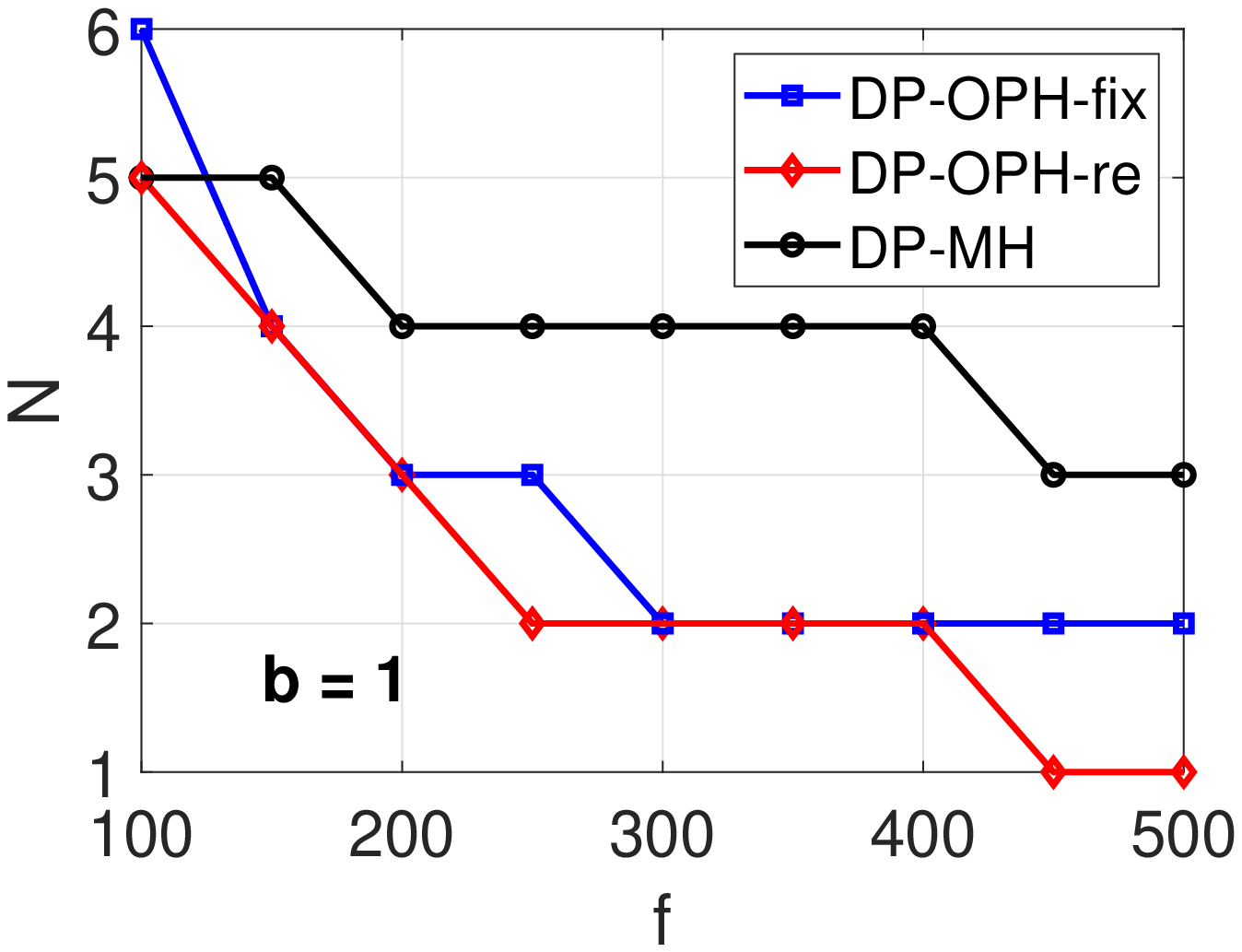} \hspace{-0.15in}
    \includegraphics[width=2.3in]{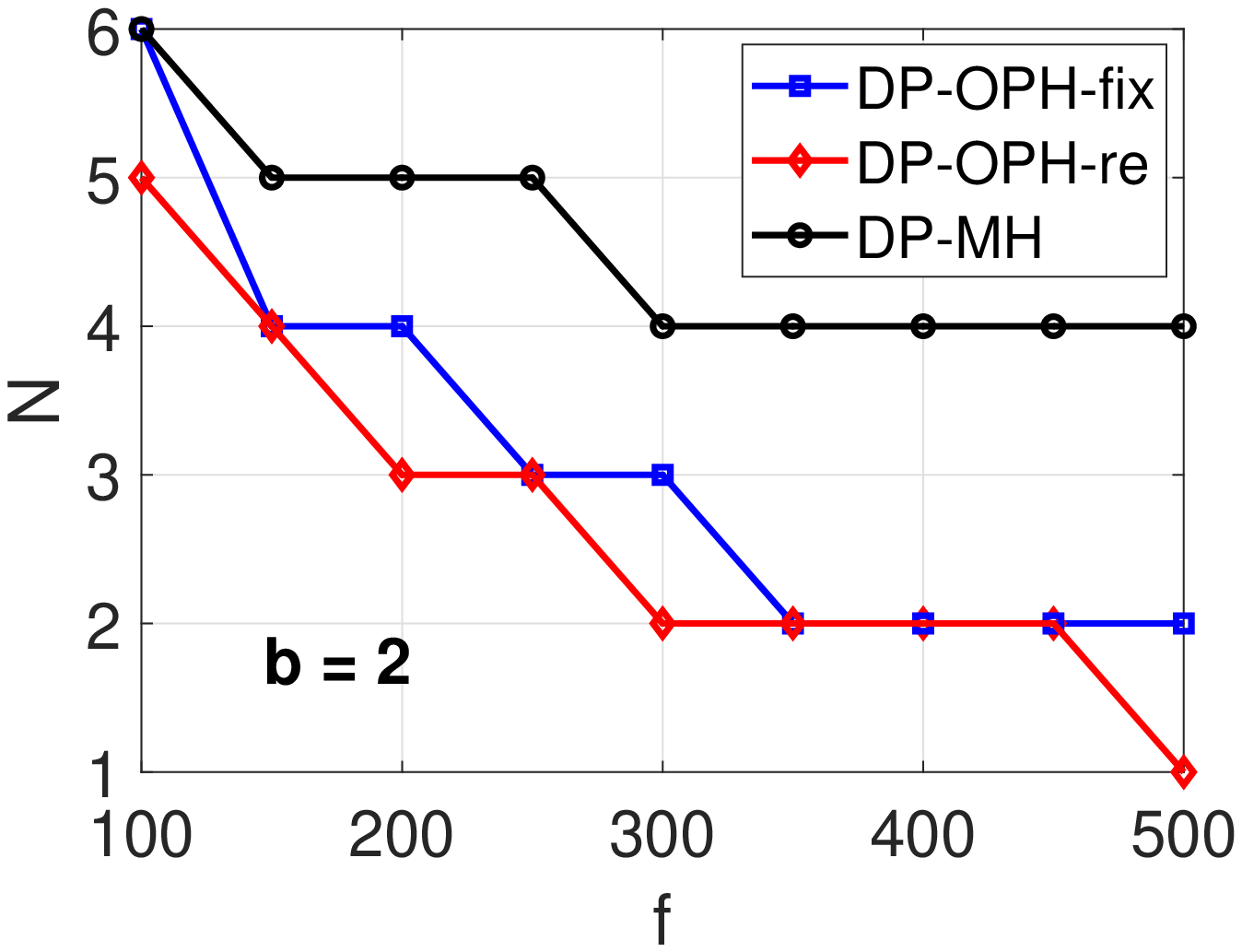} \hspace{-0.15in}
    \includegraphics[width=2.3in]{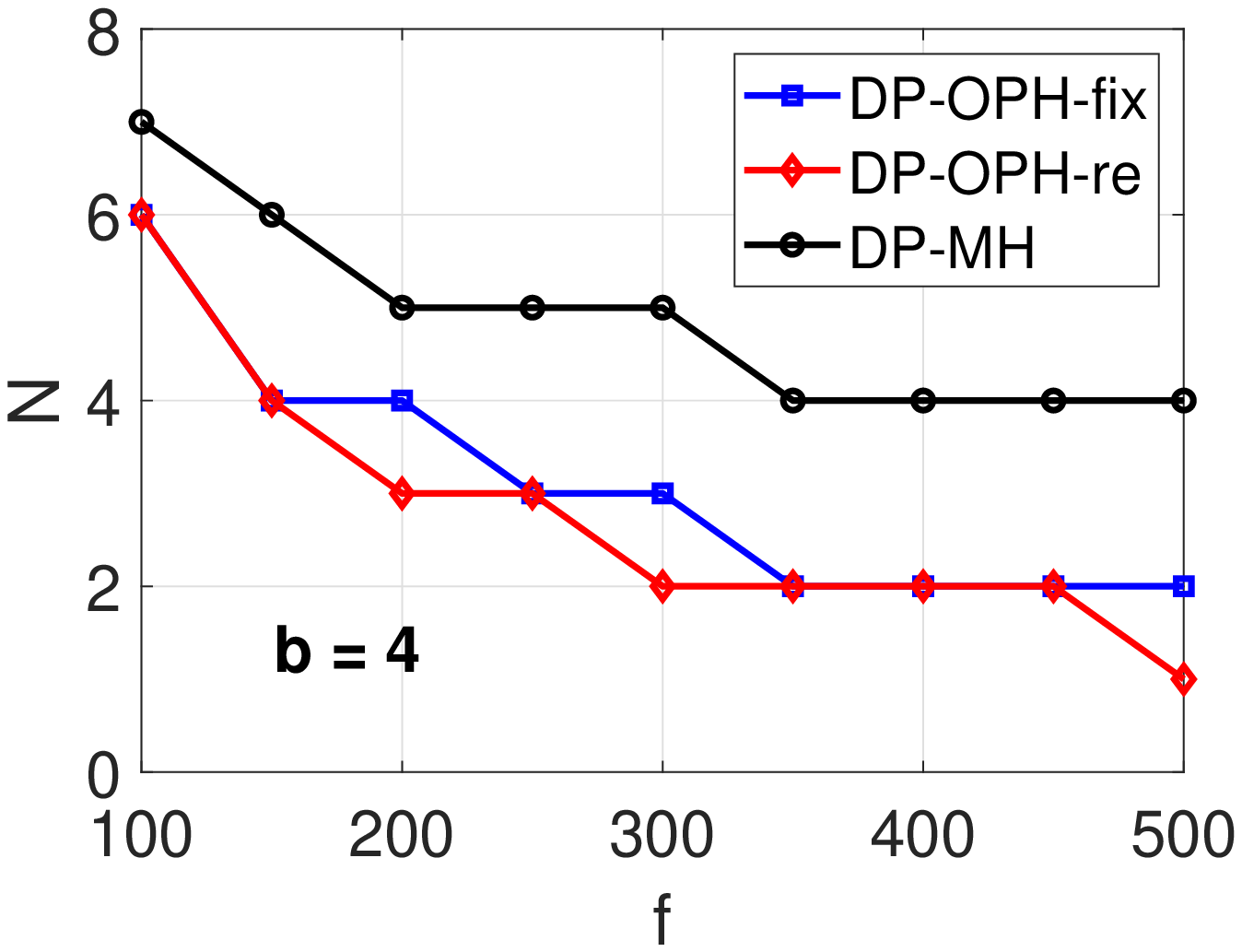} 
    }
    
\vspace{-0.15in}

\caption{Comparison of the privacy discount factor $N$ for densified DP-OPH and DP-MH, against the number of non-zero elements in the data vector $f$. $D=1024, K=64,\delta=10^{-6}$.}
\label{fig:N}
\end{figure}

We further compare the Jaccard estimation accuracy. For the two densified DP-OPH variants, DP-OPH-fix and DP-OPH-re, and the DP MinHash (DP-MH) methods, each full-precision (and unprivatized) hash value of $h(\bm u)$ and $h(\bm v)$ has collision probability equal to $P(h(\bm u)=h(\bm v))=J(\bm u,\bm v)$. Let $h^{(b)}(\bm u)$ denote the $b$-bit hash values. Since we assume the lowest $b$ bits are uniformly assigned, we have $P(h^{(b)}(\bm u)=h^{(b)}(\bm v))=J+(1-J)\frac{1}{B}$. Denote $p=\frac{\exp(\epsilon/N)}{\exp(\epsilon/N)+2^b-1}$. By simple probability calculation~\citep{aumuller2020differentially}, the privatized $b$-bit hash values has collision probability
\begin{align*}
    &P(\tilde h(\bm u)=\tilde h(\bm v))\\
    &=P(\tilde h(\bm u)=\tilde h(\bm v)|h^{(b)}(\bm u)=h^{(b)}(\bm v))P(h^{(b)}(\bm u)=h^{(b)}(\bm v)) \\
    &\hspace{2in} + P(\tilde h(\bm u)=\tilde h(\bm v)|h^{(b)}(\bm u)\neq h^{(b)}(\bm v))P(h^{(b)}(\bm u)\neq h^{(b)}(\bm v)) \\
    &=\left[ p^2+\frac{(1-p)^2}{2^b-1} \right]\left( \frac{1}{2^b}+\frac{2^b-1}{2^b}J \right) + \left[ \frac{2p(1-p)}{2^b-1} + \frac{2^b-2}{(2^b-1)^2}(1-p)^2 \right]\left( \frac{2^b-1}{2^b}-\frac{2^b-1}{2^b}J \right),
\end{align*}
which implies $J=\frac{(2^b-1)(2^b P(\tilde h(u)=\tilde h(v))-1)}{(2^b p -1)^2}$. Therefore, let $\hat J=\frac{1}{K}\sum_{k=1}^K \mathbbm 1\{\tilde h_k(u)=\tilde h_k(v)\}$, then an unbiased estimator of $J$ is 
\begin{align} \label{eqn:unbiased-est}
    \hat J_{unbias}=\frac{(2^b-1)(2^b \hat J-1)}{(2^b p -1)^2}.
\end{align}

To compare the mean squared error (MSE), we simulate two data vectors with $D=1024, K=64$, and $J=1/3$. In Figure~\ref{fig:MSE}, we vary $f$, the number of non-zeros per data vector, and report the empirical MSE of the unbiased estimator (\ref{eqn:unbiased-est}) for DP-OPH-fix, DP-OPH-re, and DP-MH, respectively. As we can see, the comparison is consistent with the comparison of $N$ in Figure~\ref{fig:N}, that the proposed DP-OPH-re has the smallest MSE among the three competitors. This again justifies the advantage of DP-OPH-re. with re-randomized densification.

\begin{figure}[h]
\centering
    \mbox{\hspace{-0.15in}
    \includegraphics[width=2.3in]{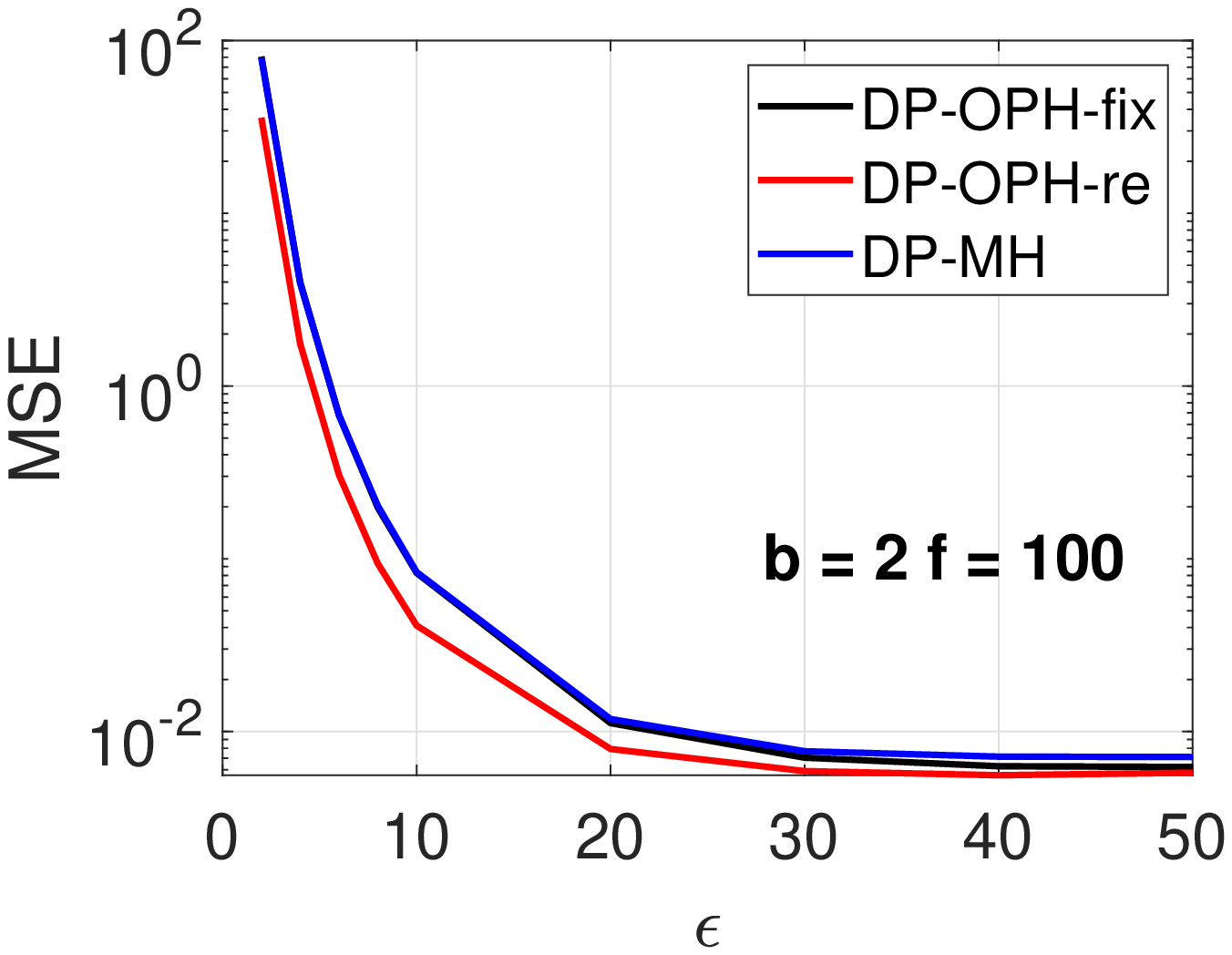} \hspace{-0.15in}
    \includegraphics[width=2.3in]{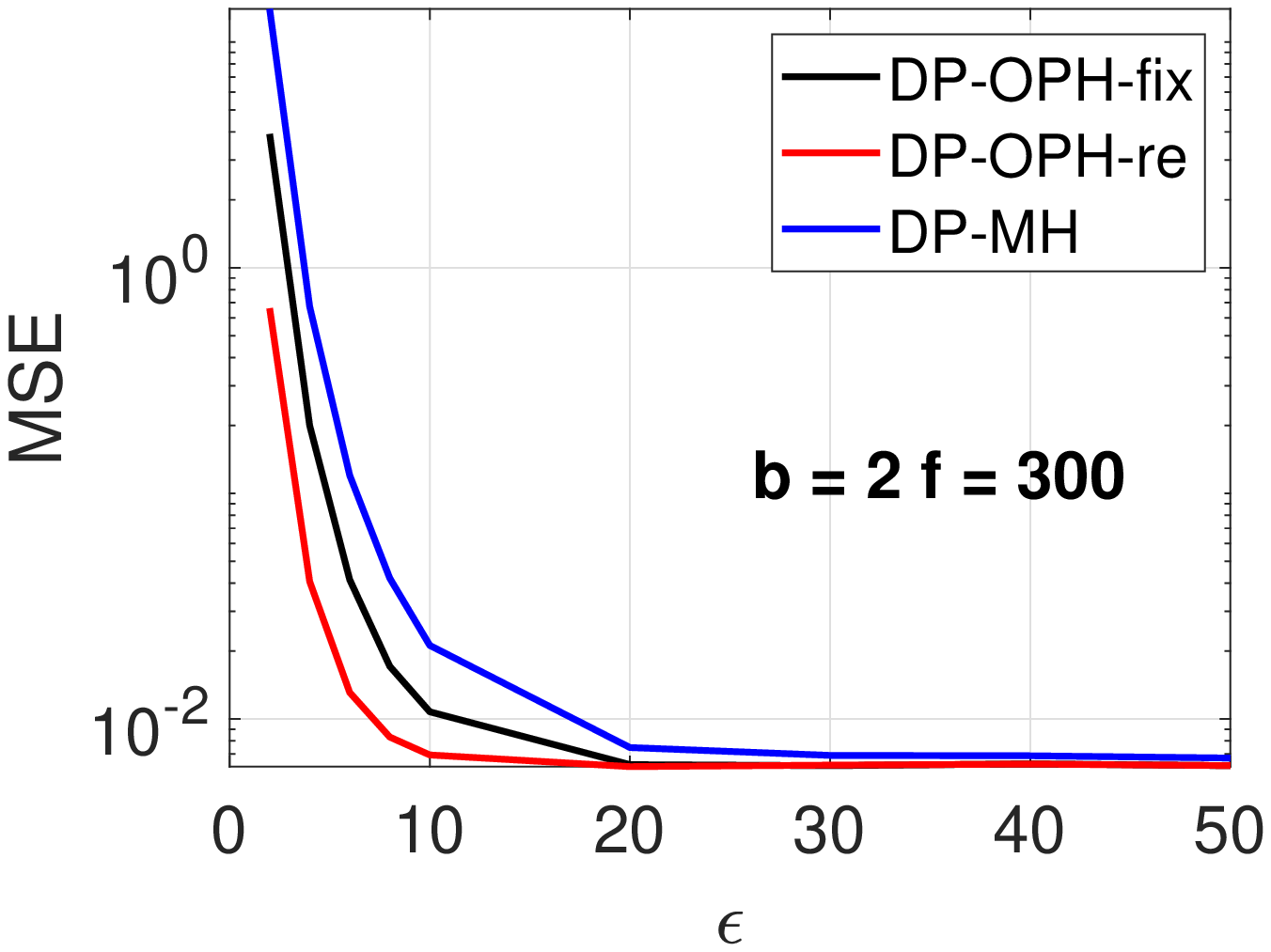} \hspace{-0.15in}
    \includegraphics[width=2.3in]{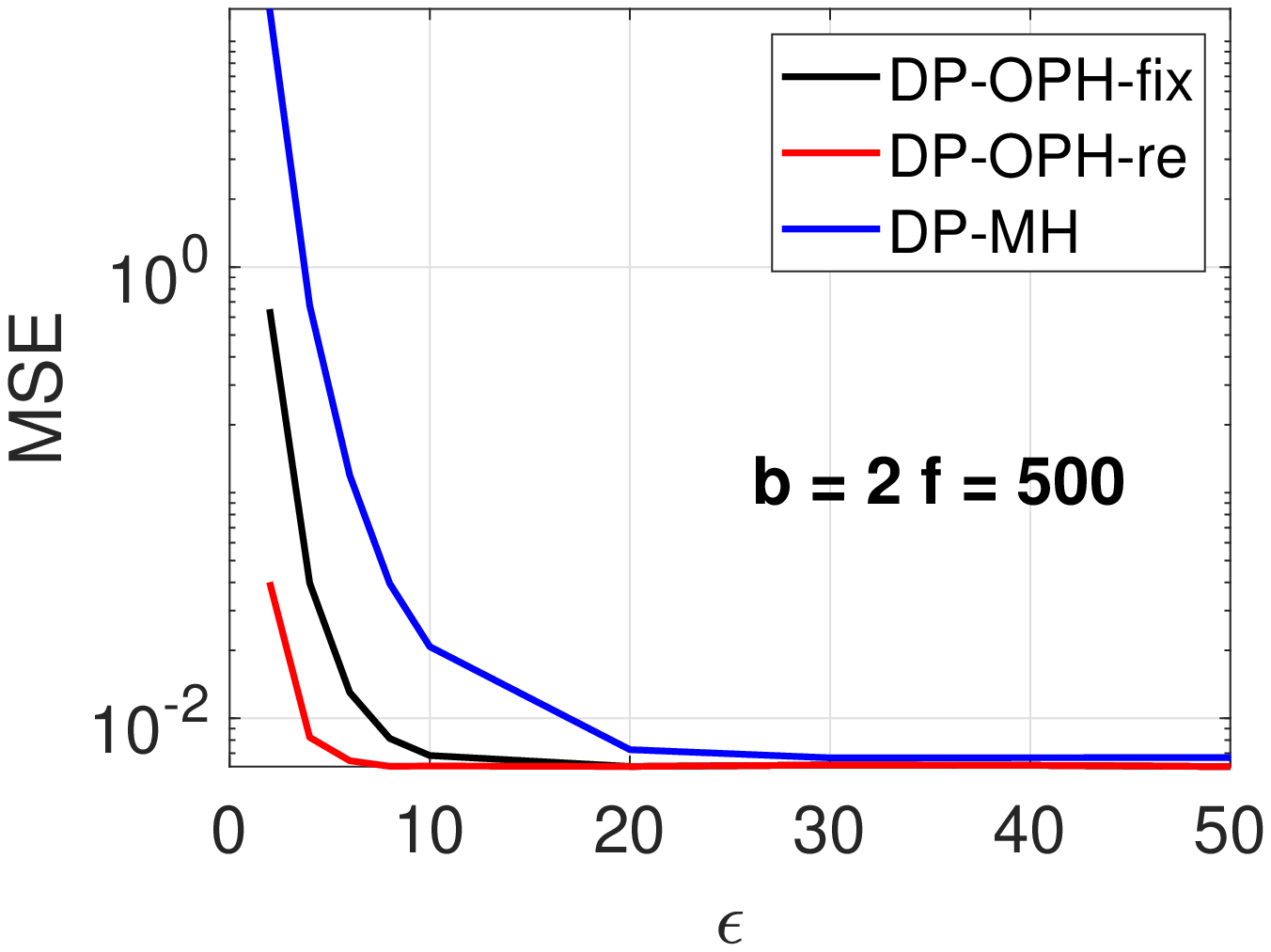} 
    }
\vspace{-0.2in}
\caption{Empirical MSE comparison of the unbiased Jaccard estimator (\ref{eqn:unbiased-est}) from DP-OPH-fix, DP-OPH-re and DP-MH. $D=1024, K=64,\delta=10^{-6}$.}
\label{fig:MSE}
\end{figure}

\section{Experiments}  \label{sec:experiment}

We conduct retrieval experiments to evaluate the proposed DP hashing methods on two public datasets: the MNIST~\citep{lecun1998mnist} hand-written digit dataset and the Webspam~\citep{chang2011libsvm} dataset. Both datasets are binarized  by setting the non-zero entries to 1. We use the train set as the database for retrieval candidates, and the test set as queries. For each query point, we set the ground truth (``gold-standard'' ) neighbors as the top 50 data points in the database with the highest Jaccard similarity to the query. To search with DP-OPH and DP-MH, we generate the private hash values and estimate the Jaccard similarity between the query and the data points using the collision estimator in the same form as (\ref{MH-estimator}). Then, we retrieve the data points with the highest estimated Jaccard similarity to the query. The evaluation metrics are precision and recall, defined as
\begin{align*}
    &\text{precision}@R=\frac{\text{\# of TPs in}\ R \ \text{retrieved points}}{R}, \quad \text{recall}@R=\frac{\text{\# of TPs in}\ R \ \text{retrieved points}}{\text{\# of gold-standard neighbors}}.
\end{align*}
For densified DP-OPH (Algorithm~\ref{alg:DP-OPH-densification}) and DP-MH (Algorithm~\ref{alg:DP-MinHash}), we ensure the lower bound $f$ on the number of non-zero elements by filtering the data points with at least $f$ non-zeros. For MNIST, we set $f=50$, which covers $99.9\%$ of the total sample points. For Webspam, we test $f=500$, which includes $90\%$ of the data points. We average the precision and recall metrics over all the query points. The reported results are averaged over 5 independent runs.

\newpage

\begin{figure}[t]

    \mbox{\hspace{-0.15in}
    \includegraphics[width=2.3in]{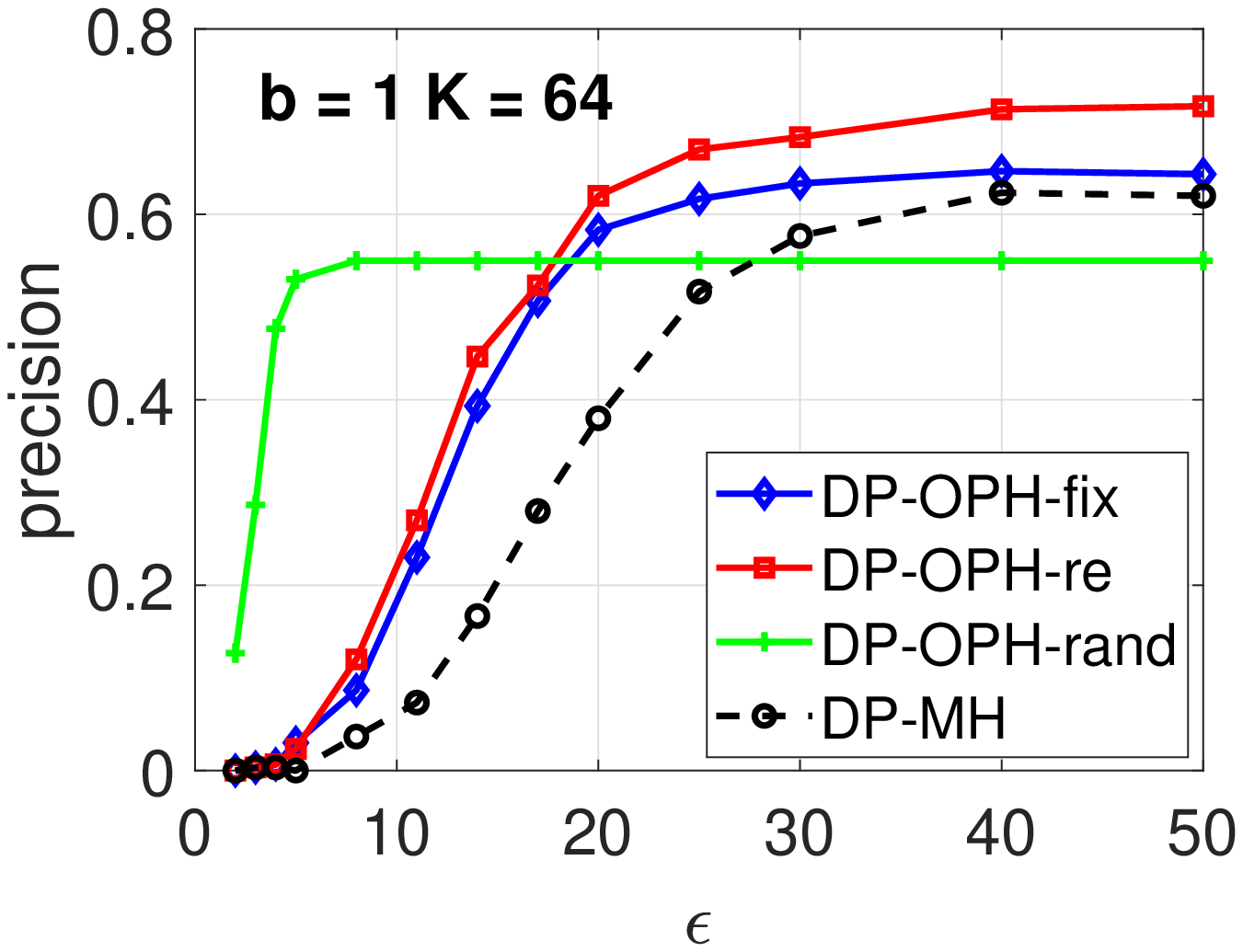} \hspace{-0.15in}
    \includegraphics[width=2.3in]{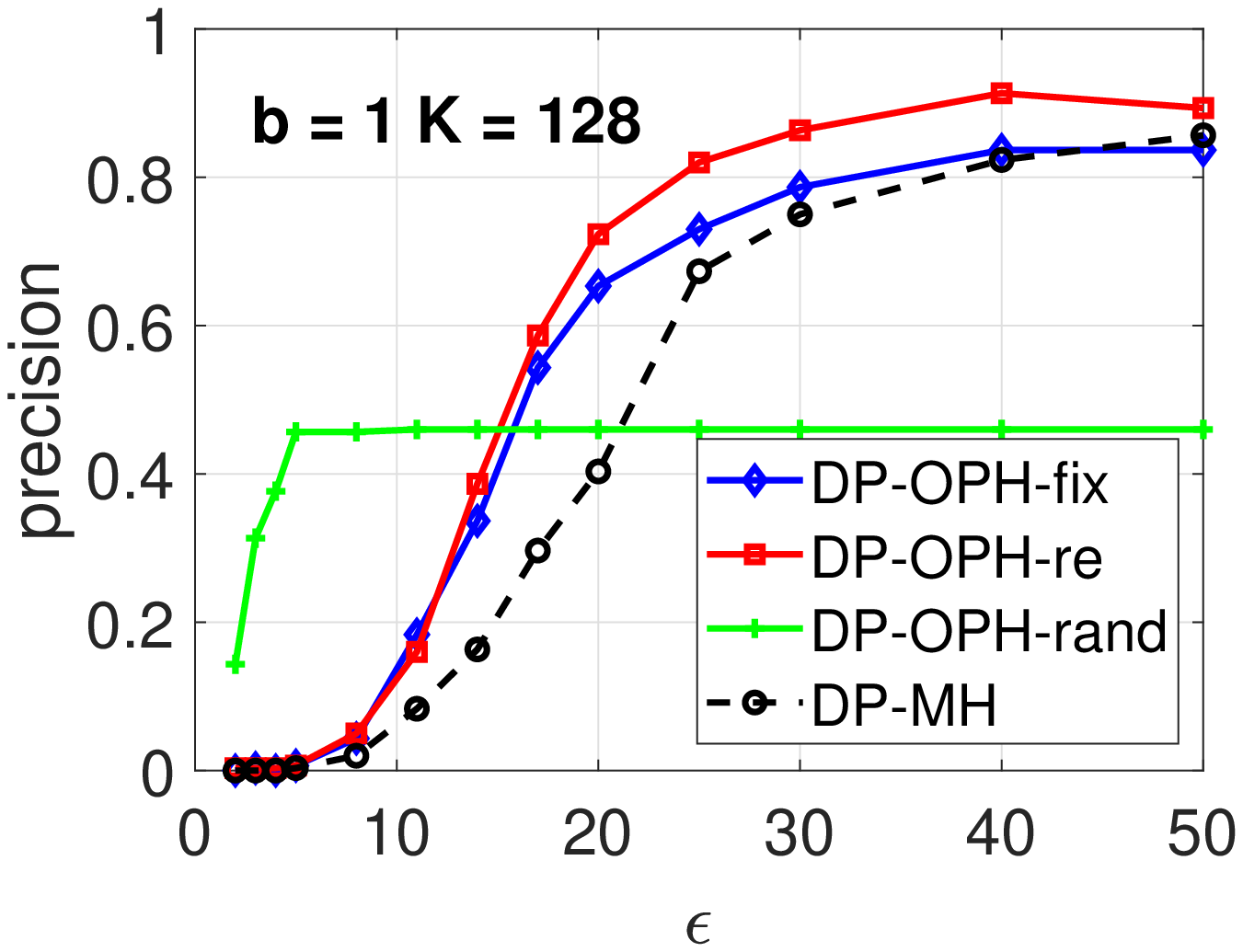} \hspace{-0.15in}
    \includegraphics[width=2.3in]{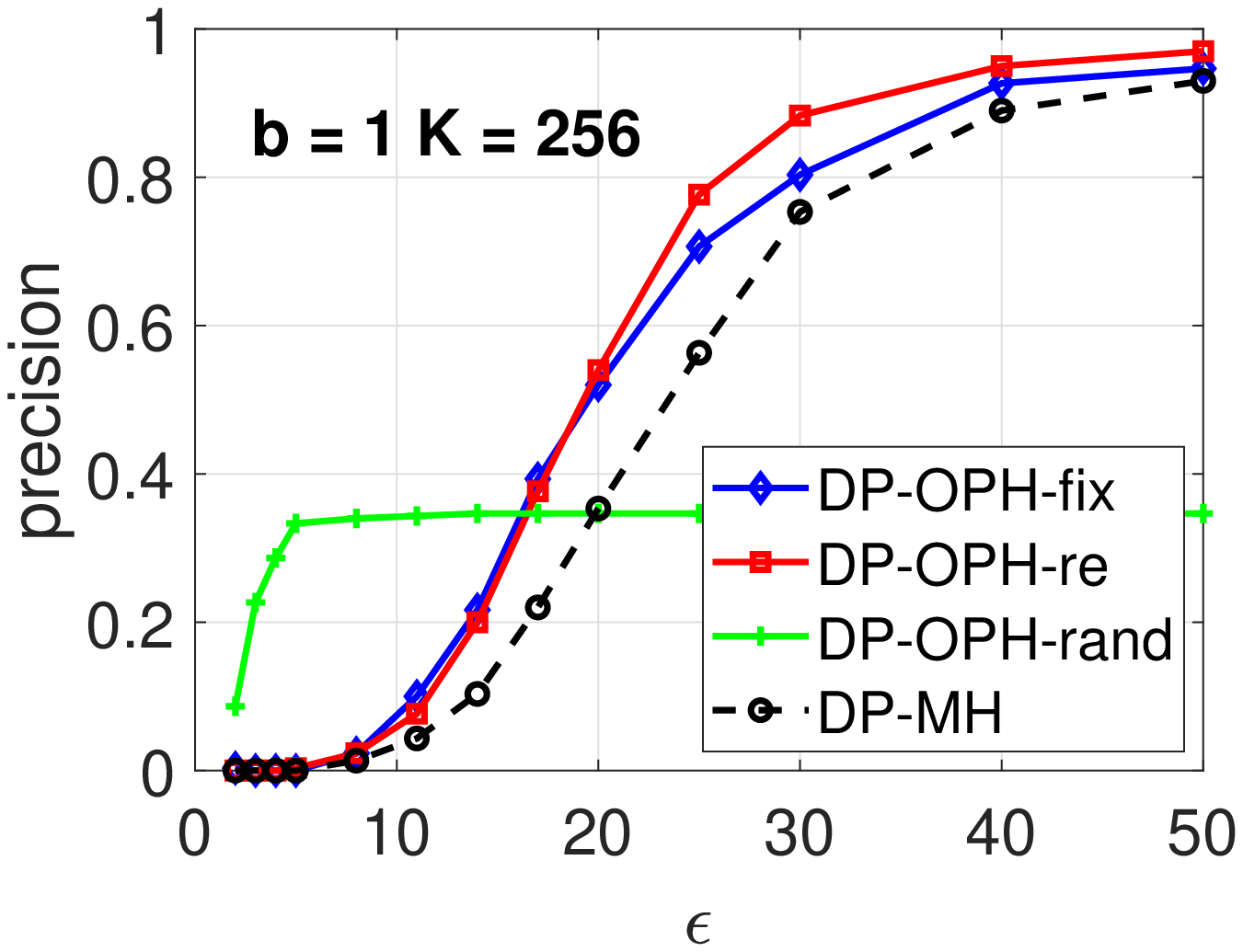}
    }
    
    \mbox{\hspace{-0.15in}
    \includegraphics[width=2.3in]{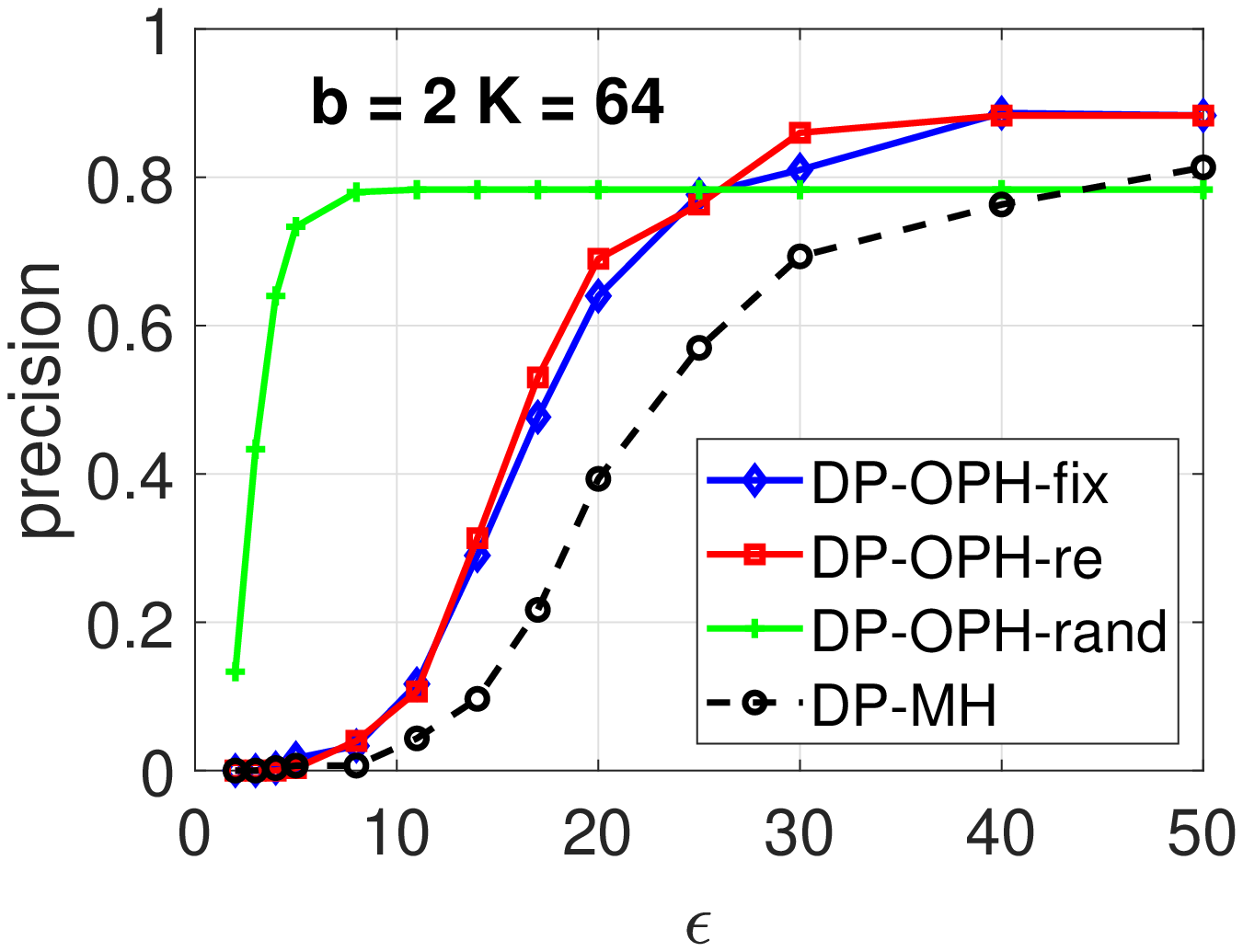} \hspace{-0.15in}
    \includegraphics[width=2.3in]{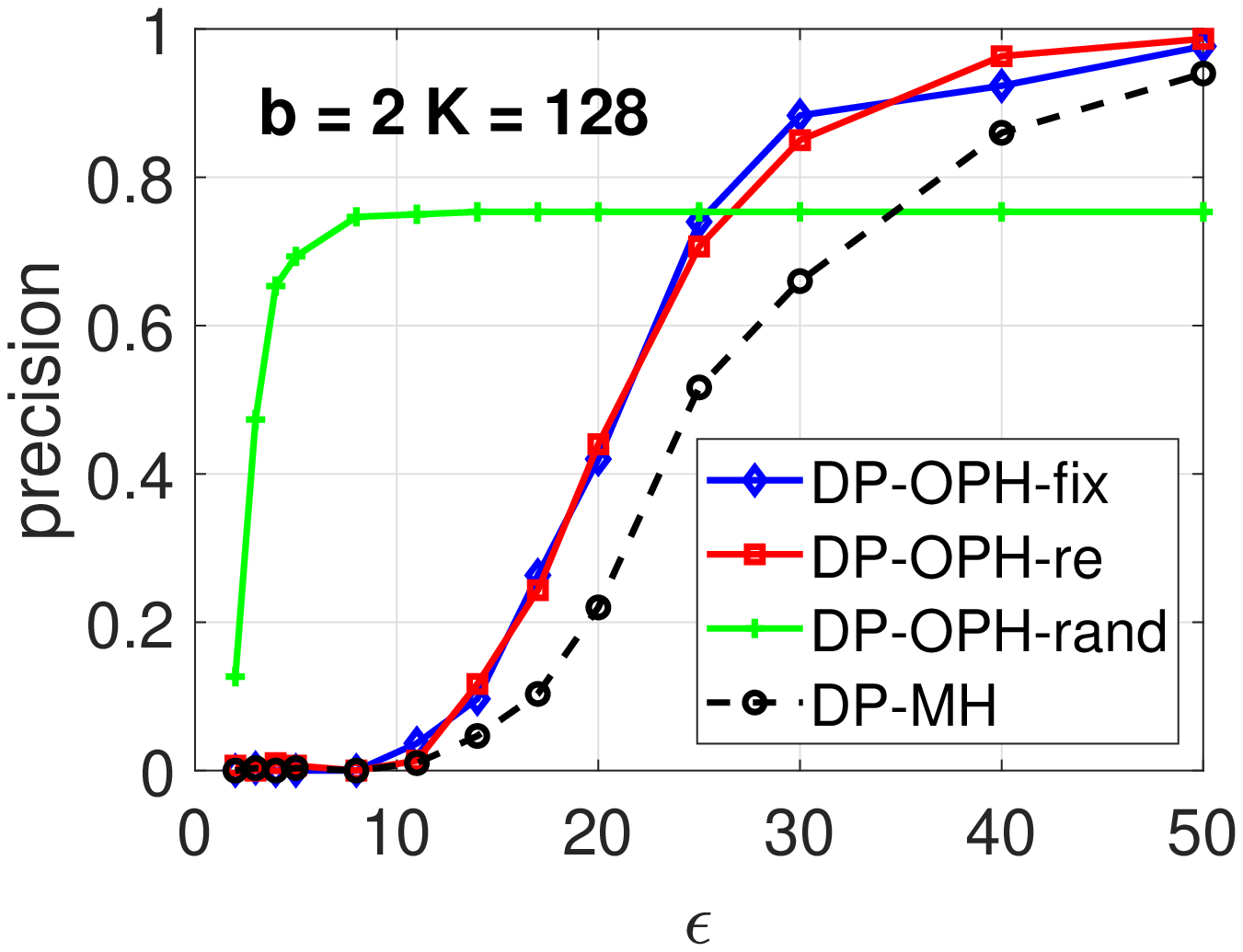} \hspace{-0.15in}
    \includegraphics[width=2.3in]{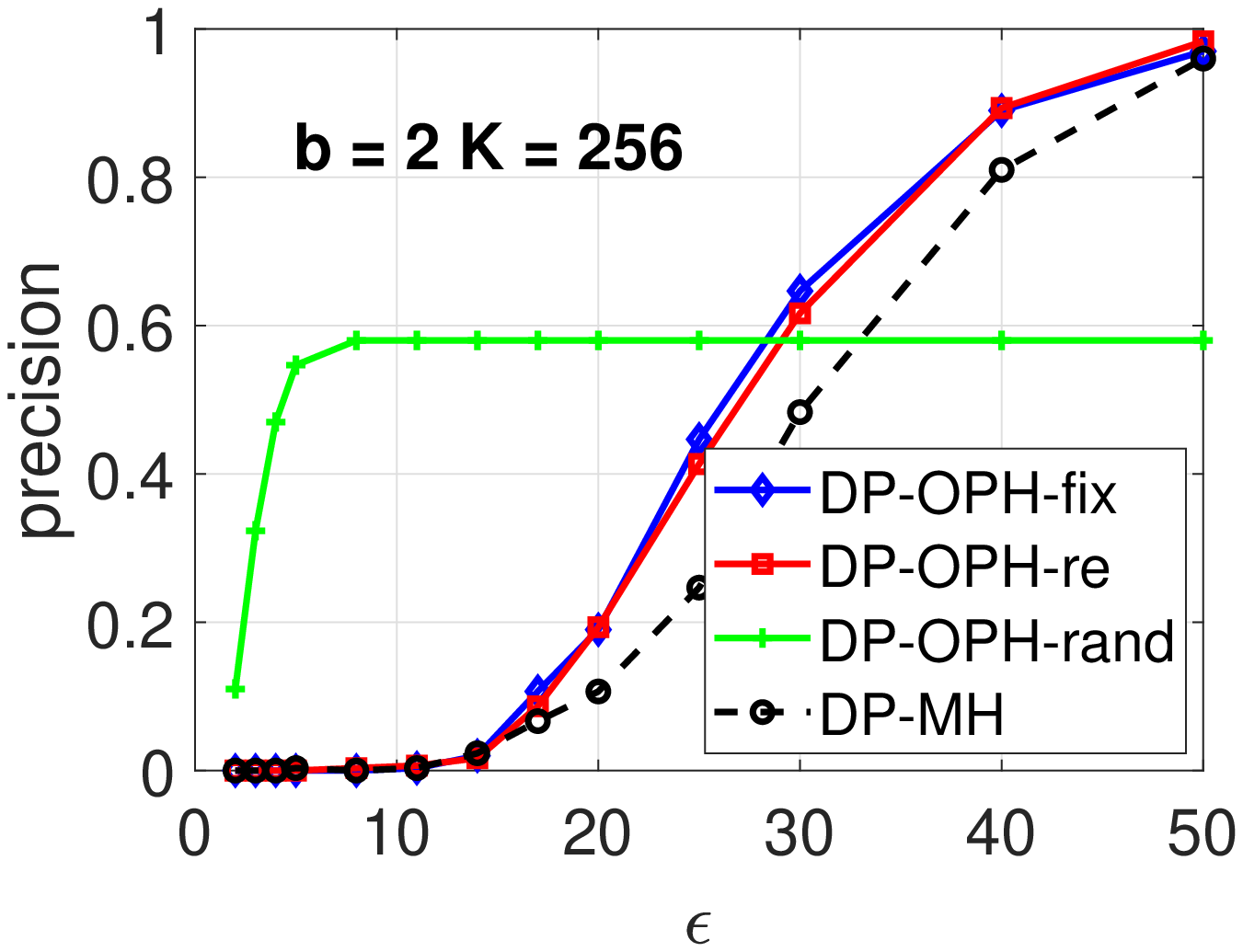}
    }

\vspace{-0.15in}

\caption{Precision@10 results on MNIST for DP-OPH variants and DP-MH, with different $K$. $\delta=10^{-6}$. First row: $b=1$. Second row: $b=2$.}
\label{fig:MNIST-b1b2}
\end{figure}

\begin{figure}[t]

    \mbox{\hspace{-0.15in}
    \includegraphics[width=2.3in]{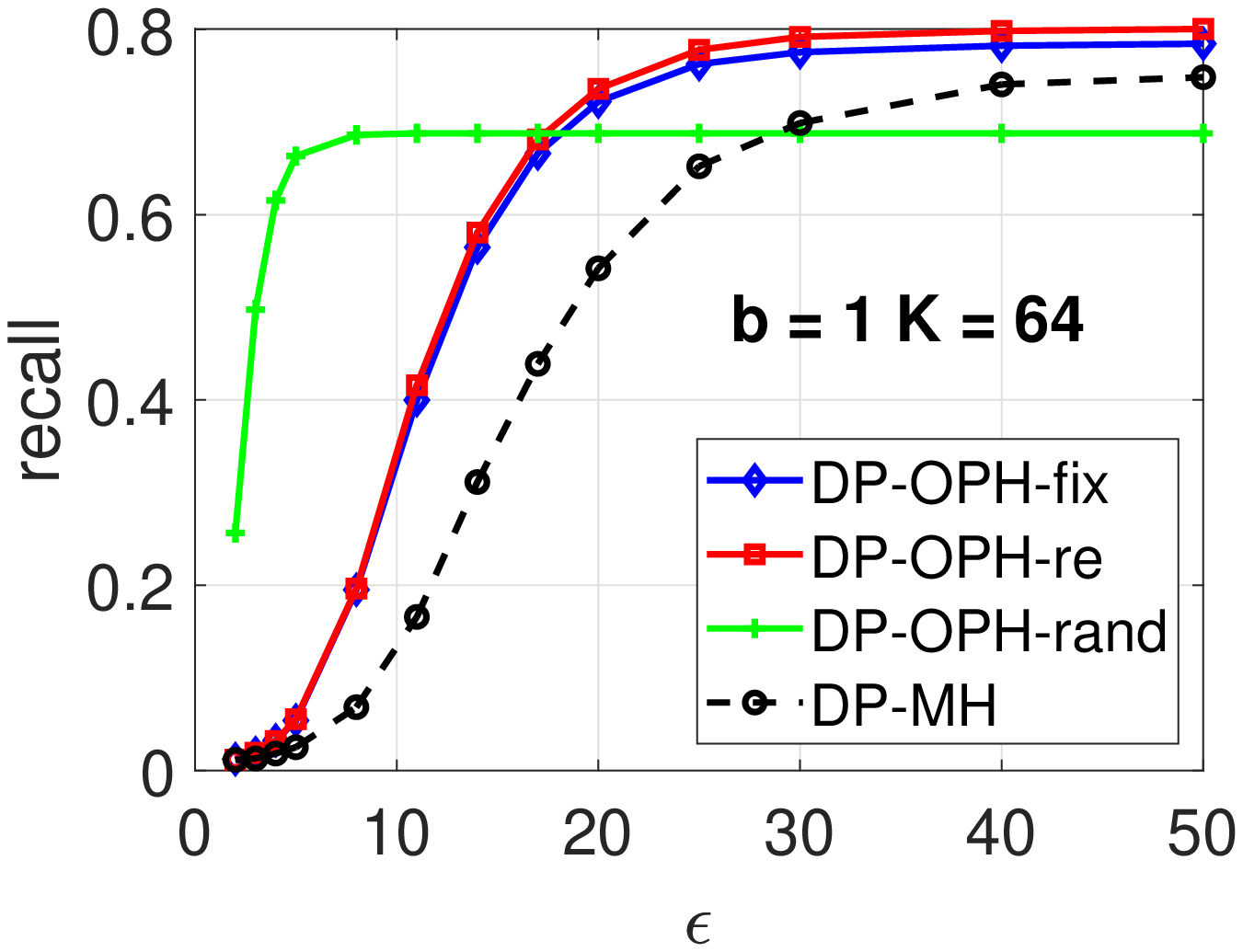} \hspace{-0.15in}
    \includegraphics[width=2.3in]{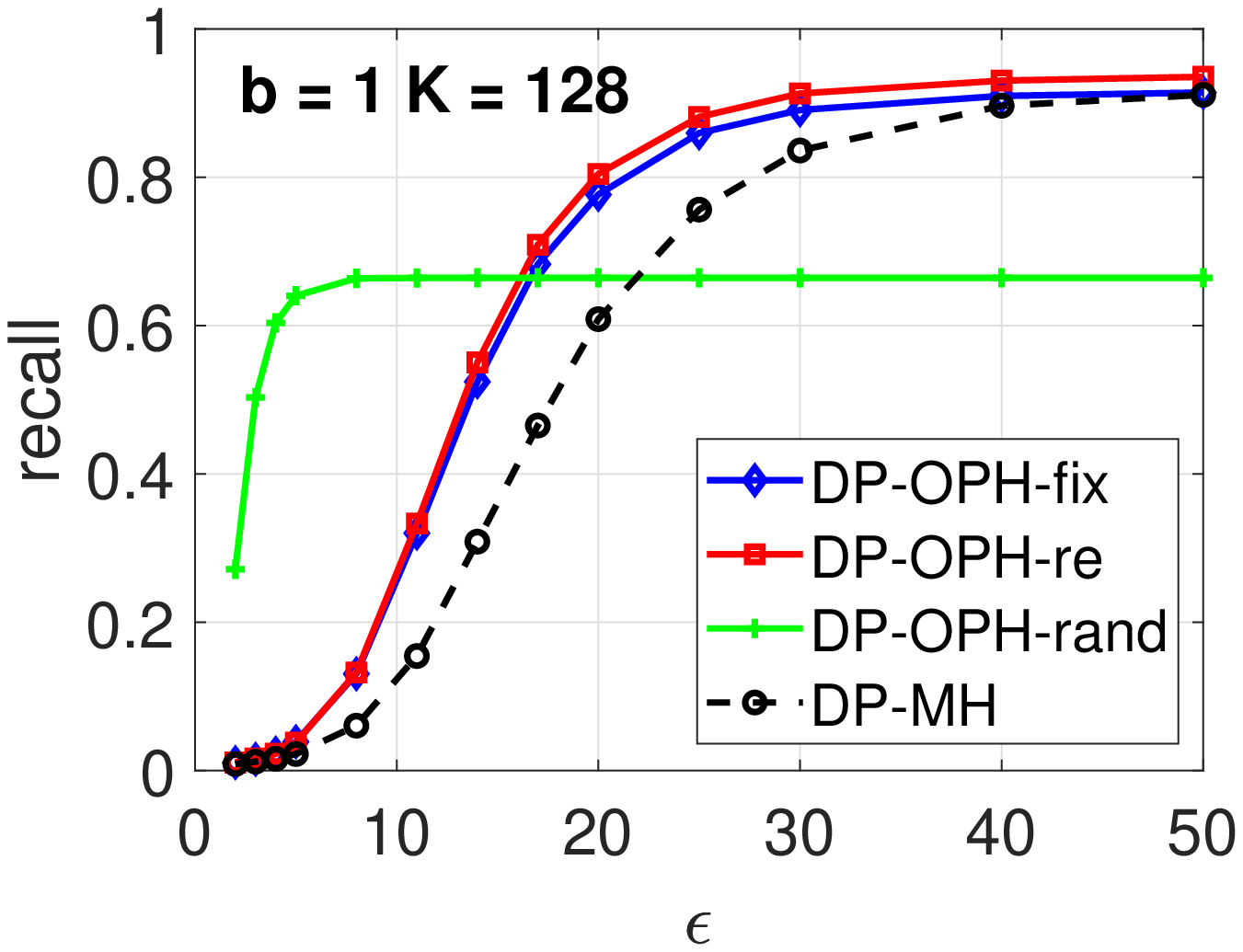} \hspace{-0.15in}
    \includegraphics[width=2.3in]{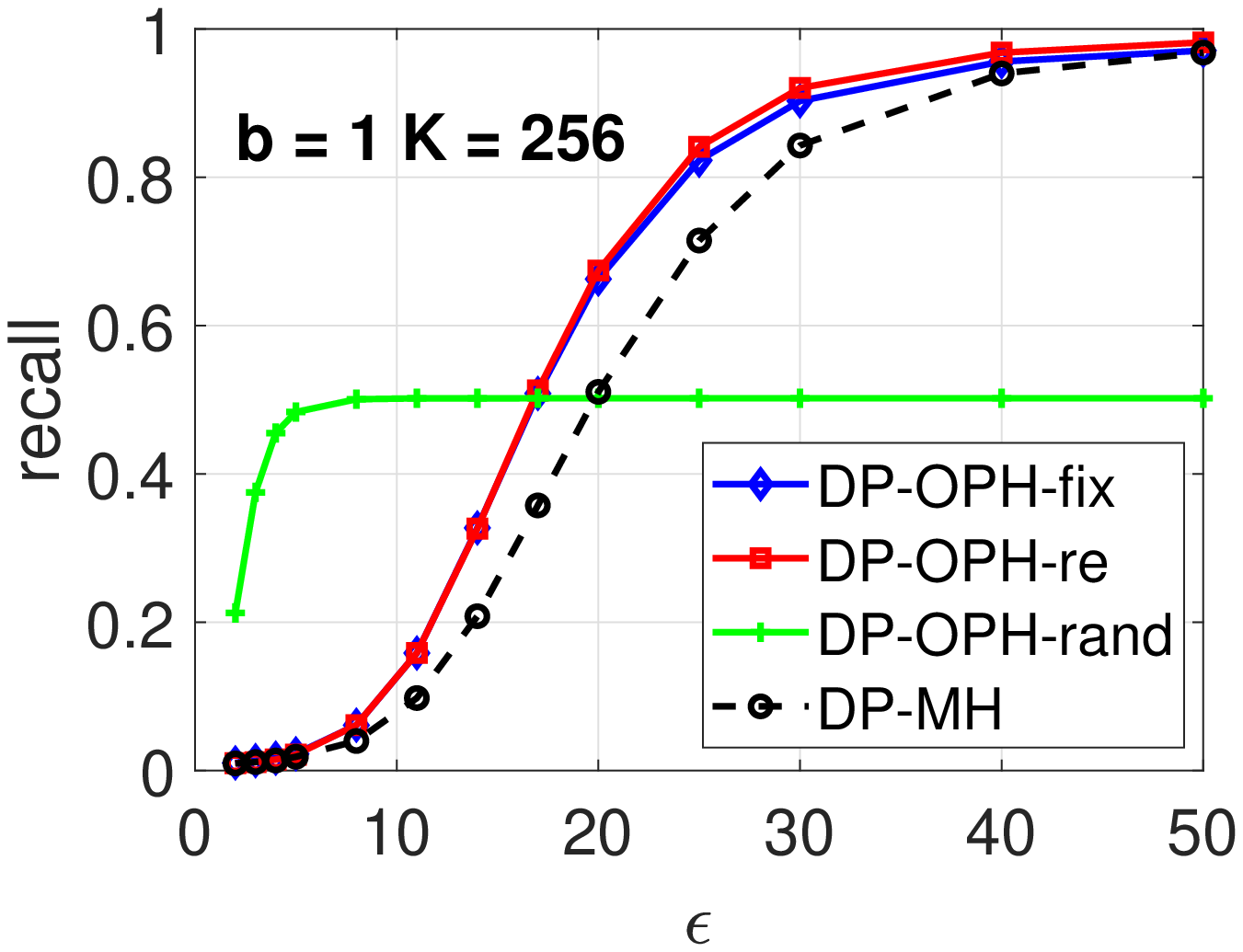}
    }
    
    \mbox{\hspace{-0.15in}
    \includegraphics[width=2.3in]{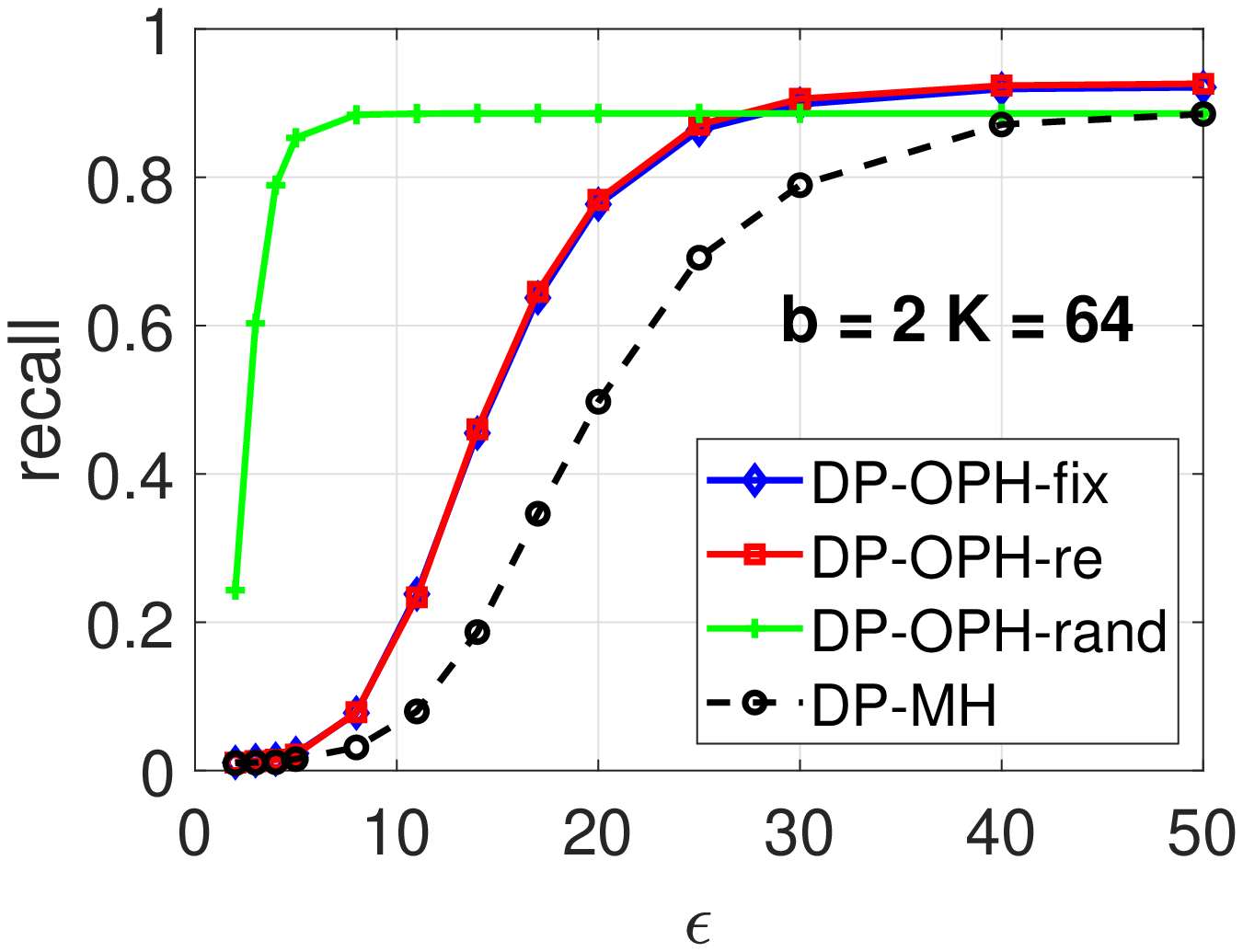} \hspace{-0.15in}
    \includegraphics[width=2.3in]{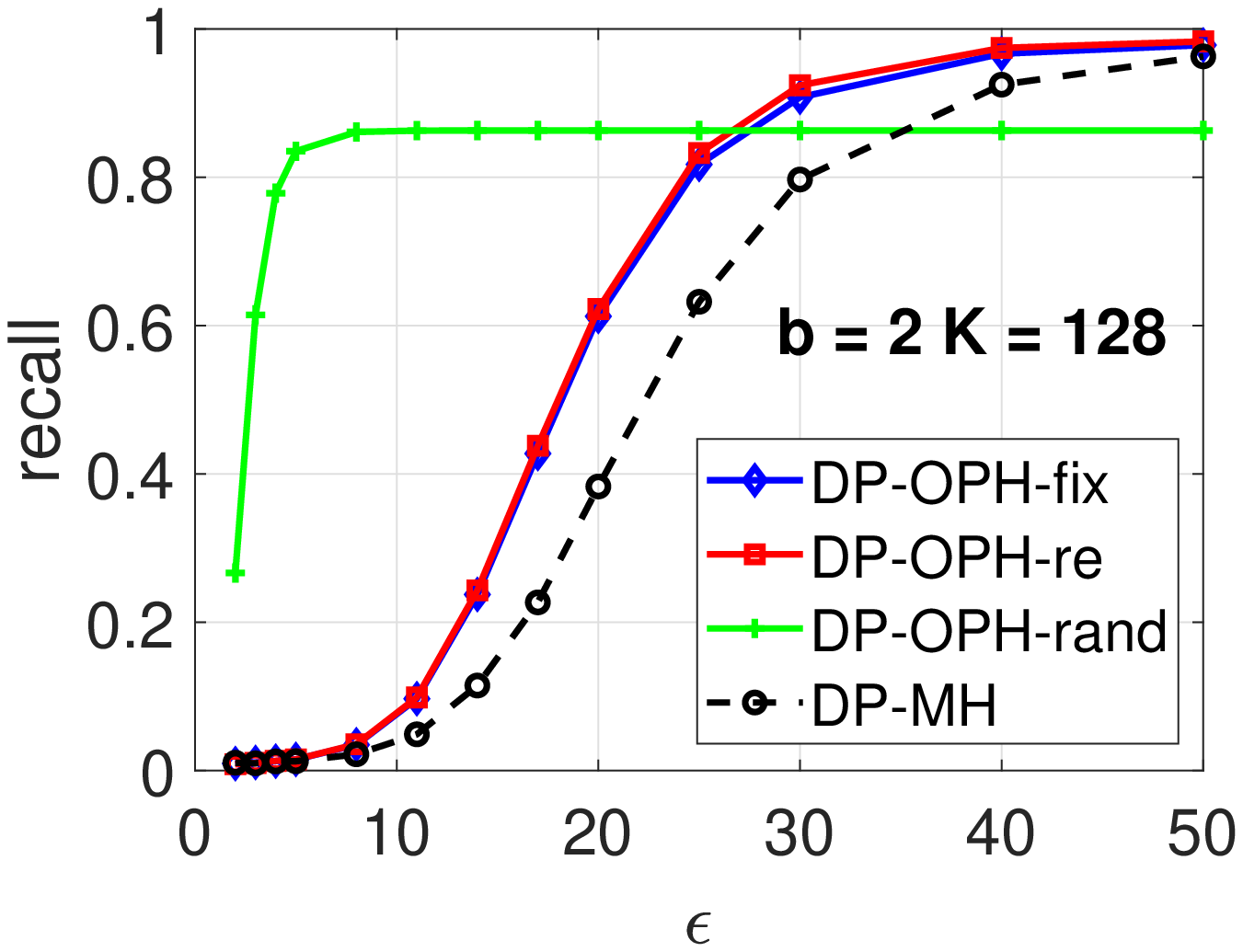} \hspace{-0.15in}
    \includegraphics[width=2.3in]{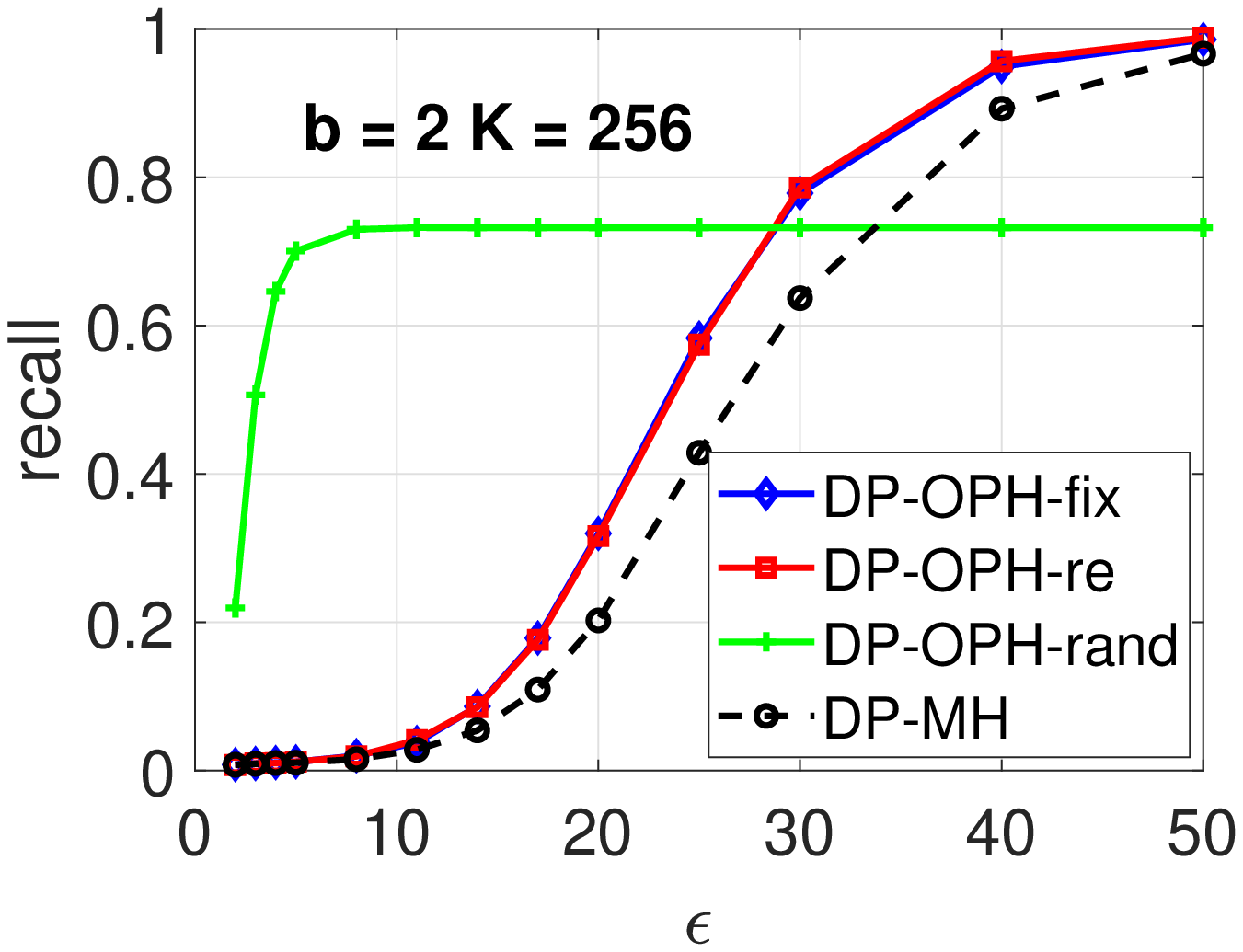}
    }

\vspace{-0.15in}

\caption{Recall@500 results on MNIST for DP-OPH variants and DP-MH, with different $K$. $\delta=10^{-6}$. First row: $b=1$. Second row: $b=2$.}
\label{fig:MNIST-b1b2-recall}
\end{figure}

\newpage\clearpage

\noindent\textbf{Results.} In Figure~\ref{fig:MNIST-b1b2} and Figure~\ref{fig:MNIST-b1b2-recall}, we plot the precision and recall respectively on MNIST for 1-bit and 2-bit DP hashing methods. We see that: 
\begin{itemize}
    \item Consistent with Figure~\ref{fig:N}, DP-OPH-re performs considerably better than DP-MH and DP-OPH-fix, for all $\epsilon$ levels.
    
    \item DP-OPH-rand achieves impressive search accuracy with small $\epsilon$ (e.g., $\epsilon<5$), but stops improving with $\epsilon$ afterward (due to the random bits for the empty bins), which demonstrates the trade-off discussed in Section~\ref{sec:DP-OPH}. When $\epsilon$ gets larger (e.g., $\epsilon>15$), DP-OPH-rand is outperformed by DP-OPH-re.
    
    \item Increasing $b$ is relatively more beneficial for DP-OPH-rand as it can achieve higher search accuracy with small $\epsilon$. Also, larger $\epsilon$ is required for DP-OPH-re to bypass DP-OPH-rand.
\end{itemize}

The results on Webspam with $b=2$ are presented in Figure~\ref{fig:webspam-b2}. Again, DP-OPH-re achieves better performance than DP-MH and DP-OPH-fix for all $\epsilon$, especially when $\epsilon<20$. DP-OPH-rand performs the best with $\epsilon<10$. Yet, it may underperform when $\epsilon$ increases.

\begin{figure}[h]

    \mbox{\hspace{-0.15in}
    \includegraphics[width=2.3in]{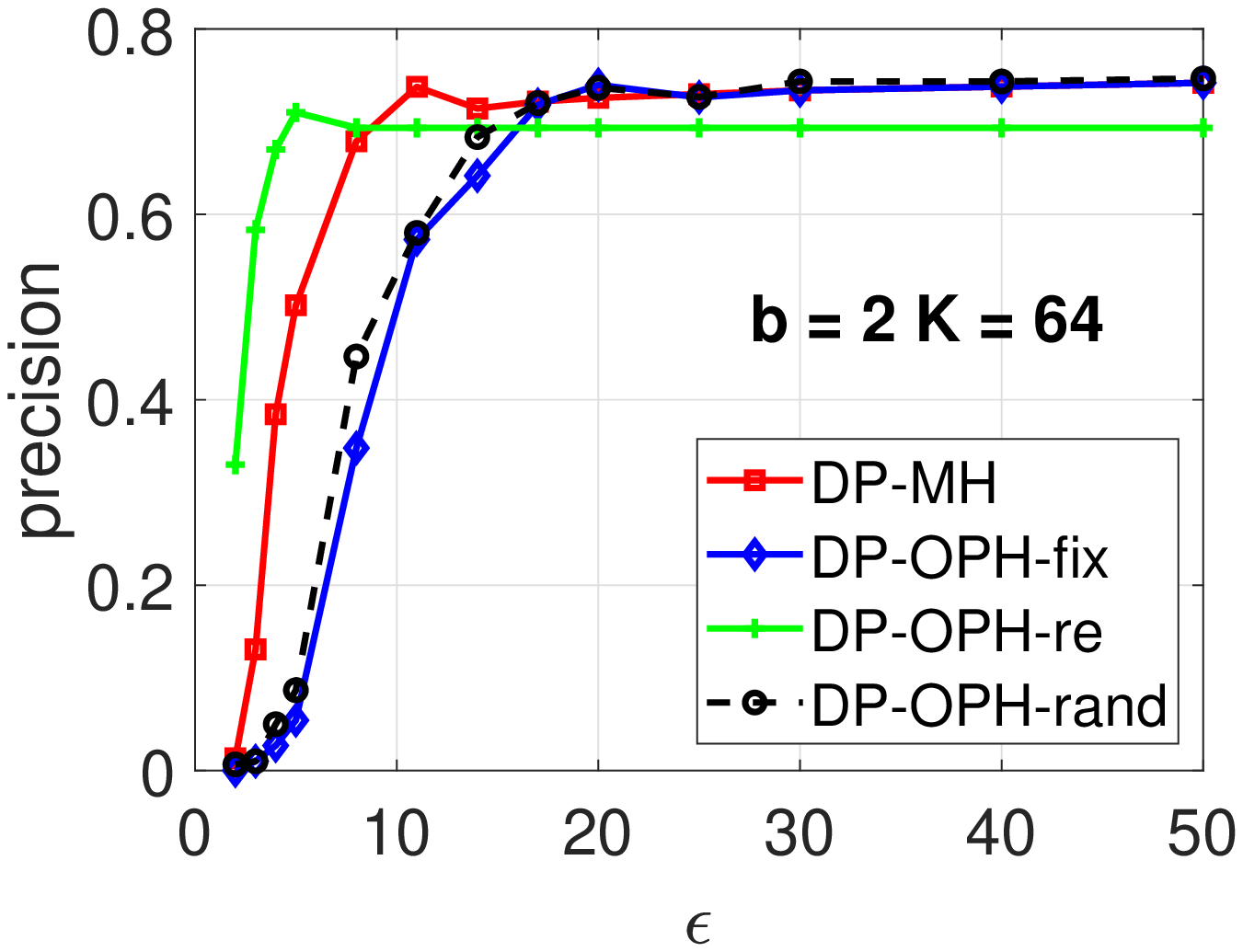} \hspace{-0.15in}
    \includegraphics[width=2.3in]{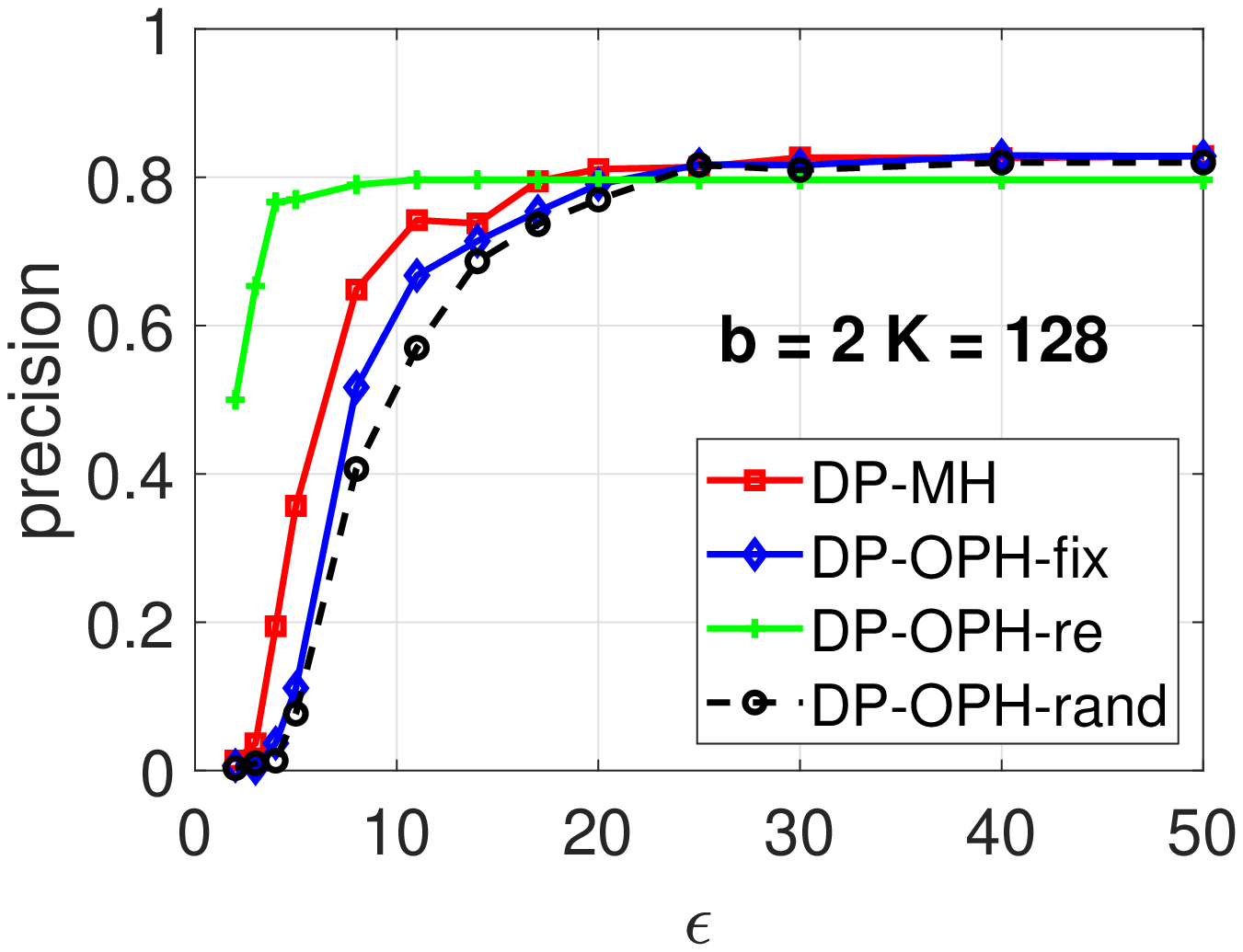} \hspace{-0.15in}
    \includegraphics[width=2.3in]{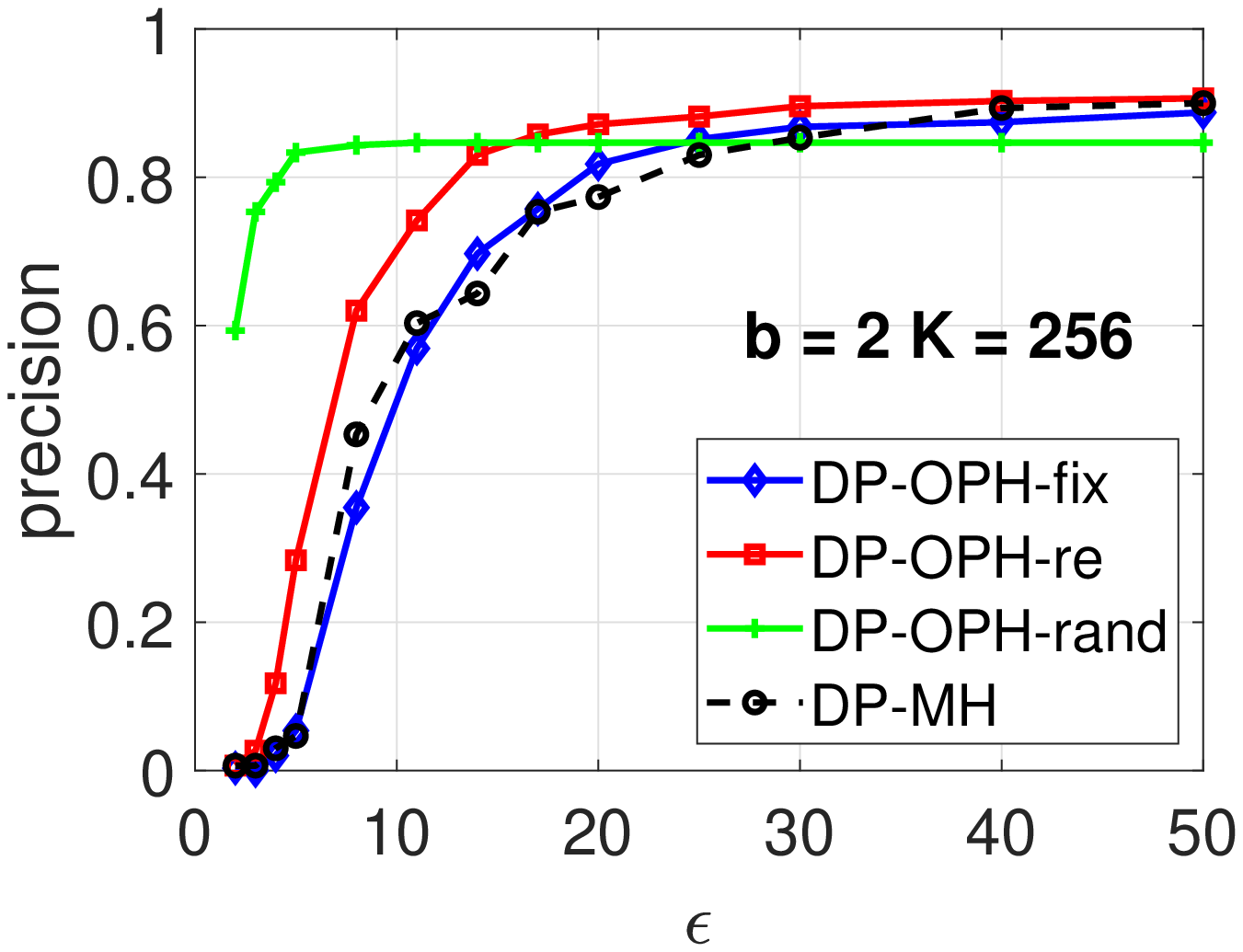}
    }
    \mbox{\hspace{-0.15in}
    \includegraphics[width=2.3in]{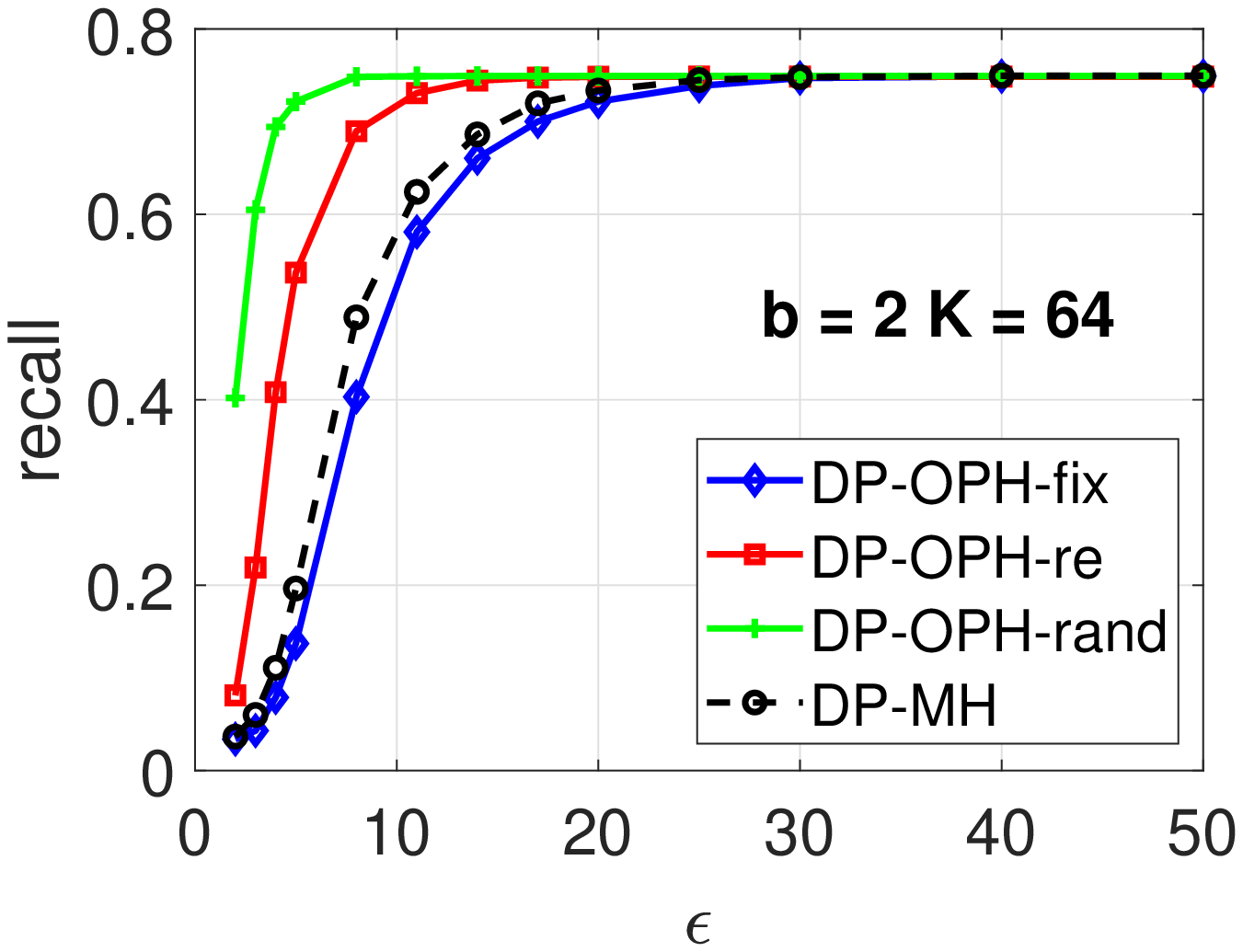} \hspace{-0.15in}
    \includegraphics[width=2.3in]{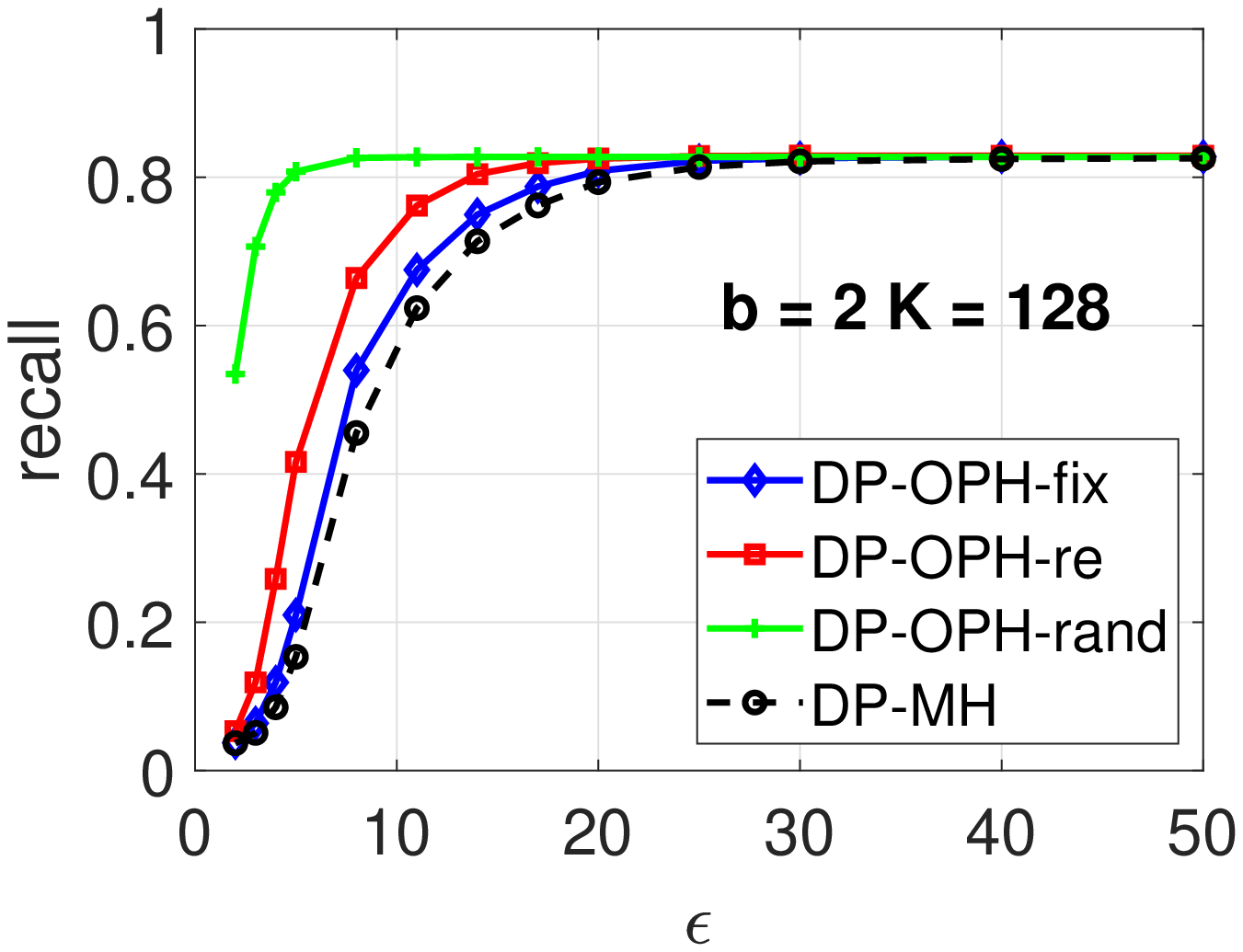} \hspace{-0.15in}
    \includegraphics[width=2.3in]{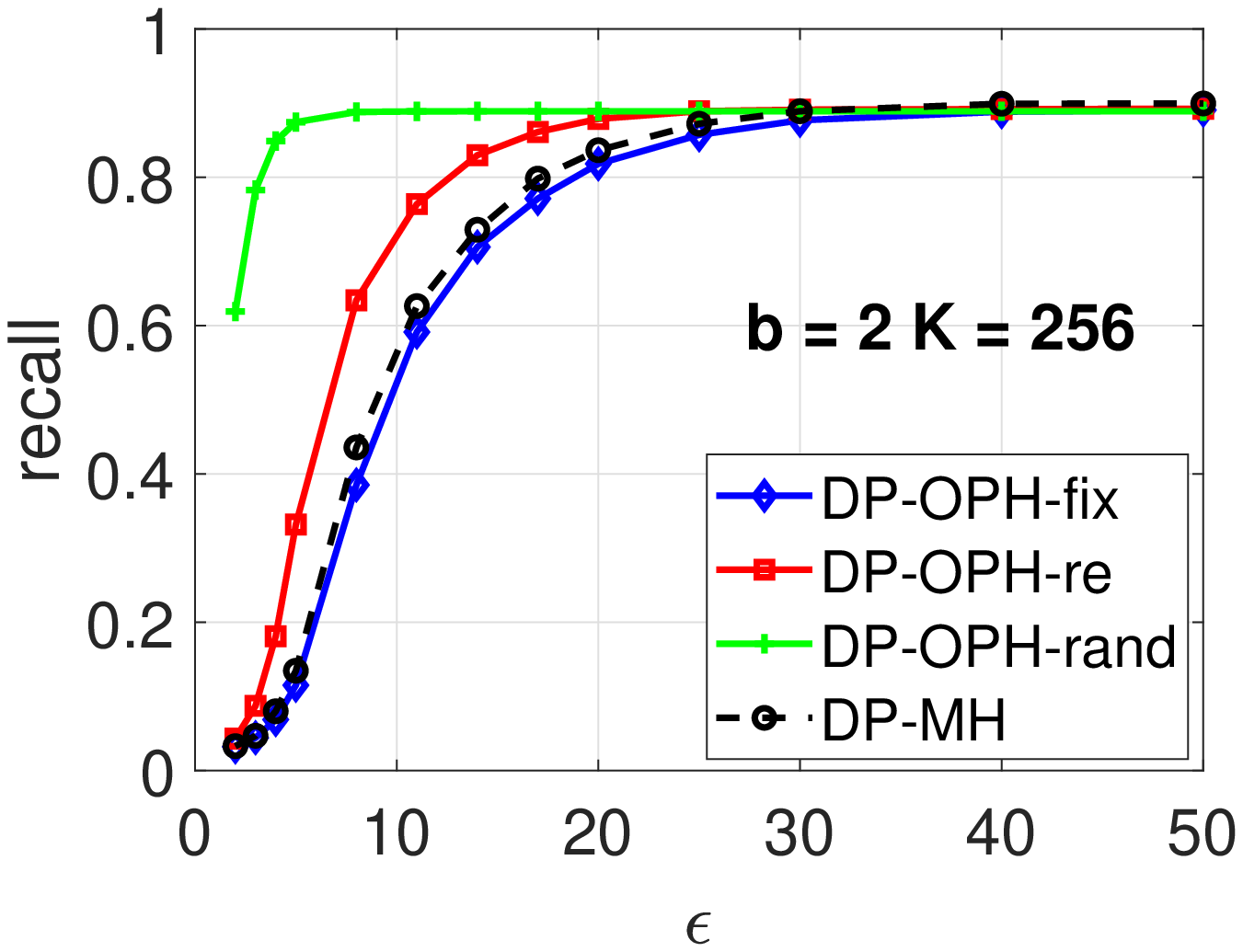}
    }
    
\vspace{-0.15in}
\caption{Precision@10 and recall@500 results on Webspam for 2-bit hash values. $\delta=10^{-6}$.}
\label{fig:webspam-b2}
\end{figure}

In summary, our experiments show that, the proposed DP-OPH-re with re-randomized densification is better than DP-MH and DP-OPH-fix with fixed densification for all $\epsilon$. When $\epsilon$ is small (e.g., $\epsilon<10$), the DP-OPH-rand variant might achieve the best practical utility. When $\epsilon$ is allowed to be larger, DP-OPH-re should be the top choice for hashing the Jaccard similarity.

\section{Extension: Differentially Private Bin-wise Consistent Weighted Sampling (DP-BCWS) for Weighted Jaccard Similarity}

\begin{algorithm}[H]
		\textbf{Input:} Non-negative data vector $\bm u\in \mathbb R_+^D$
		
		\textbf{Output:} Consistent weighted sampling hash $h^*=(i^*, t^*)$
		
		\begin{algorithmic}[1]	
		
		\For {every non-zero $v_i$}  
		
		\State $r_i\sim Gamma(2, 1)$, \ $c_i\sim Gamma(2, 1)$,  $\beta_i\sim Uniform(0, 1)$
		
		\State $t_i\leftarrow \lfloor \frac{\log u_i }{r_i}+\beta_i\rfloor$, \ \  $y_i\leftarrow \exp(r_i(t_i-\beta_i))$

        \State $a_i\leftarrow c_i/(y_i\exp(r_i))$
		
		\EndFor
		
		\State $i^* \leftarrow arg\min_i \ a_i$,\hspace{0.3in}  $t^* \leftarrow t_{i^*}$
		
		\end{algorithmic}
    \caption{Consistent Weighted Sampling (CWS)}
    \label{alg:CWS}
\end{algorithm}

\vspace{0.1in}

In our main paper, we focused on DP hashing algorithms for the binary Jaccard similarity. Indeed, our algorithm can also be extended to hashing the weighted Jaccard similarity: (recall the definition)
\begin{align}\label{eqn:weighted-Jaccard}
J_w({\bm u,\bm v})=\frac{\sum_{i=1}^D\min\{u_i,v_i \}}{\sum_{i=1}^D\max\{u_i,v_i \}},
\end{align}
for two non-negative data vectors $\bm u,\bm v\in\mathbb R_+$. The standard hashing algorithm for (\ref{eqn:weighted-Jaccard}) is called Consistent Weighted Sampling (CWS) as summarized in Algorithm~\ref{alg:CWS}~\citep{ioffe2010improved,manasse2010consistent,li2021consistent,li2022p}. To generate one hash value, we need three length-$D$ random vectors $\bm r\sim Gamma(2,1)$, $\bm c\sim Gamma(2,1)$ and $\bm\beta\sim Uniform(0,1)$. We denote Algorithm~\ref{alg:CWS} as a function $CWS(\bm u;\bm r,\bm c,\bm\beta)$. \cite{li2019re} proposed bin-wise CWS (BCWS) which exploits the same idea of binning as in OPH. The binning and densification procedure of BCWS is exactly the same as OPH (Algorithm~\ref{alg:OPH} and Algorithm~\ref{alg:densification}), except that every time we apply CWS, instead of MinHash, to the data in the bins to generate hash values. Note that in CWS, the output contains two values: $i^*$ is a location index similar to the output of OPH, and $t^*$ is a real-value scalar. Prior studies (e.g., \citet{li2021consistent}) showed that the second element has minimal impact on the estimation accuracy in most practical cases (i.e., only counting the collision of the first element suffices). Therefore, in our study, we also only keep the first integer element as the hash output for subsequent learning tasks.

\vspace{0.1in}

For weighted data vectors, we follow the prior DP literature on weighted sets (e.g.,~\citet{xu2013differentially,smith2020flajolet,dickens2022order,zhao2022differentially,li2023differential}) and define the neighboring data vectors as those who differ in one element. To privatize BCWS, there are also three possible ways depending on the densification option. Since the DP algorithm design for densified BCWS requires rigorous and non-trivial computations which might be an independent study, here we  empirically test the ($b$-bit) DP-BCWS method with random bits for empty bins. The details are provided in Algorithm~\ref{alg:DP-BCWS}. In general, we first randomly split the data entries into $K$ equal length bins, and apply CWS to the data $\bm u_{\mathcal B_k}$ in each non-empty bin $\mathcal B_k$ using the random numbers ($\bm r_{\mathcal B_k},\bm c_{\mathcal B_k},\bm\beta_{\mathcal B_k}$) to generated $K$ hash values (possibly including empty bins). After each hash is transformed into the $b$-bit representation, we uniformly randomly assign a hash value in $\{0,...,2^b-1\}$ to every empty bin.

\newpage

\begin{algorithm}[t]
	\textbf{Input:} Binary vector $\bm u\in\{0,1\}^D$; number of hash values $K$; number of bits per hash $b$

	\textbf{Output:} DP-BCWS hash values $\tilde h_1(\bm u),...,\tilde h_K(\bm u)$
	
\begin{algorithmic}[1]	
    \State Generate length-$D$ random vectors $\bm r\sim Gamma(2,1)$, $\bm c\sim Gamma(2,1)$, $\bm\beta\sim Uniform(0,1)$
    
	\State Let $d=D/K$. Use a permutation $\pi:[D]\mapsto [D]$ with fixed seed to randomly split $[D]$ into $K$ equal-size bins $\mathcal B_1,...,\mathcal B_K$, with $\mathcal B_k=\{j\in [D]:(k-1)d+1\leq \pi(j)\leq kd\}$
	
	\For{ $k=1$ to $K$}
	
	\If{Bin $\mathcal B_k$ is non-empty}
	
	\State $h_k(\bm u)\leftarrow CWS(\bm u_{\mathcal B_k};\bm r_{\mathcal B_k},\bm c_{\mathcal B_k},\bm\beta_{\mathcal B_k})$  \Comment{Run CWS within each non-empty bin}

    \State $h_k(\bm u)\leftarrow \text{last}\ b\ \text{bits of}\ h_k(\bm u)$
    \State $\tilde h_k(\bm u)=
    \begin{cases}
    h_k(\bm u), & \text{with probability}\ \frac{e^{\epsilon}}{e^{\epsilon}+2^b-1}\\
    i, & \text{with probability}\ \frac{1}{e^{\epsilon'}+2^b-1},\ \text{for}\ i\in \{0,...,2^b-1\},\ i\neq h_k(\bm u)
    \end{cases}$
	
	\Else
	
	\State $h_k(\bm u)\leftarrow E$

      \State $\tilde h_k(\bm u)=i$ with probability $\frac{1}{2^b}$, for $i=0,...,2^b-1$  \Comment{Assign random bits to empty bin}
	\EndIf

	\EndFor
\end{algorithmic}
\caption{Differential Private Bin-wise Consistent Weighted Sampling (DP-BCWS)}
\label{alg:DP-BCWS}
\end{algorithm}

Using the same proof arguments as Theorem~\ref{theo:DP-OPH-rand}, we have the following guarantee.

\begin{theorem} \label{theo:DP-BCWS-rand}
Algorithm~\ref{alg:DP-BCWS} satisfies $\epsilon$-DP.
\end{theorem}

\vspace{0.1in}
\noindent\textbf{Empirical evaluation.} In Figure~\ref{fig:BCWS_dailysports}, we train an $l_2$-regularized logistic regression on the DailySports dataset\footnote{\url{https://archive.ics.uci.edu/ml/datasets/daily+and+sports+activities}}. and report the test accuracy with various $b$ and $K$ values. The $l_2$ regularization parameter $\lambda$ is tuned over a fine grid from $10^{-4}$ to $10$. Similar to the results in the previous section, the performance of DP-BCWS becomes stable as long as $\epsilon>5$. Note that, linear logistic regression only gives $\approx 75\%$ accuracy on original DailySports dataset (without DP). With DP-BCWS, the accuracy can reach $\approx 98\%$ with $K=1024$ and $\epsilon=5$. 

\vspace{0.05in}

\begin{figure}[h]

    \mbox{\hspace{-0.15in}
    \includegraphics[width=2.3in]{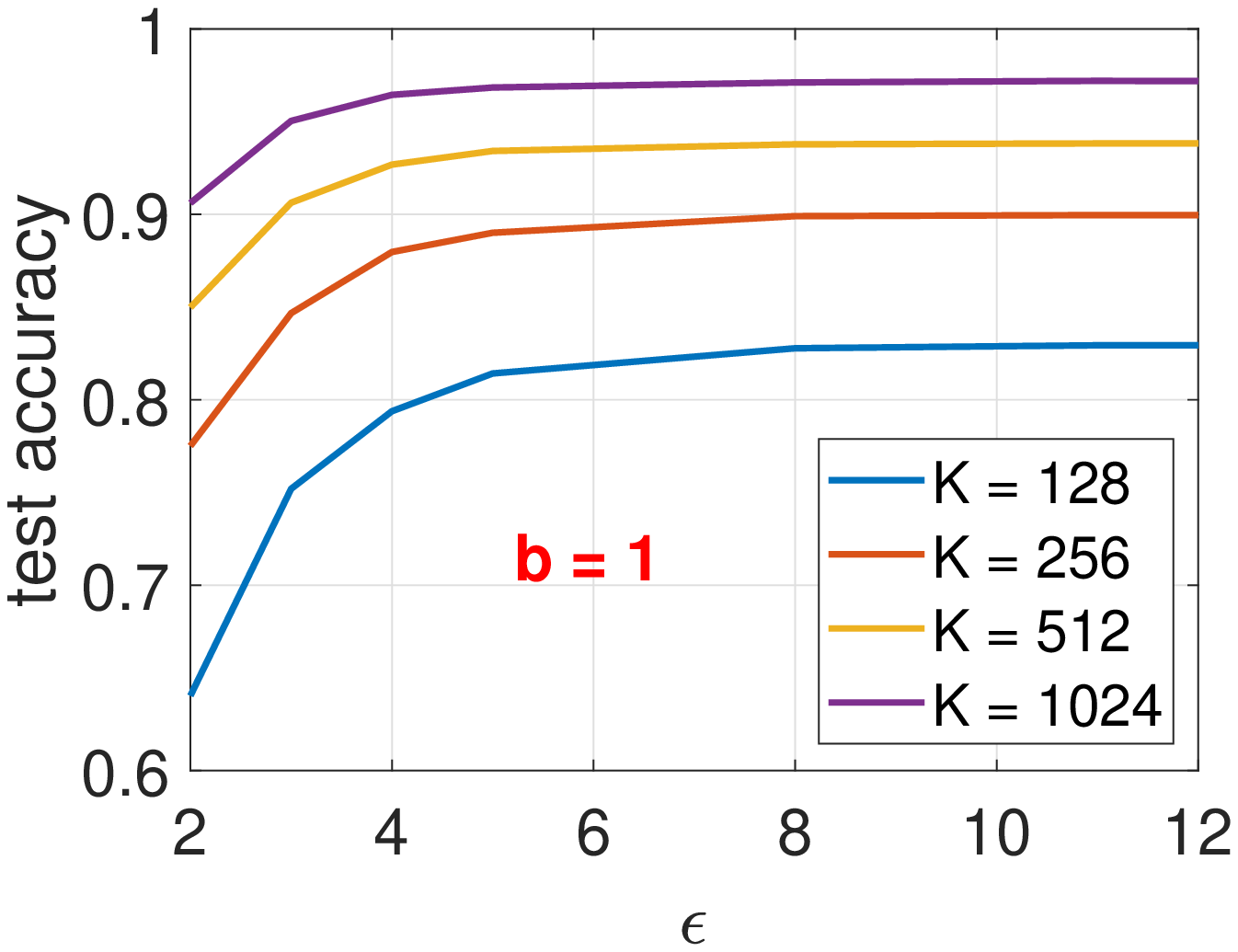} \hspace{-0.15in}
    \includegraphics[width=2.3in]{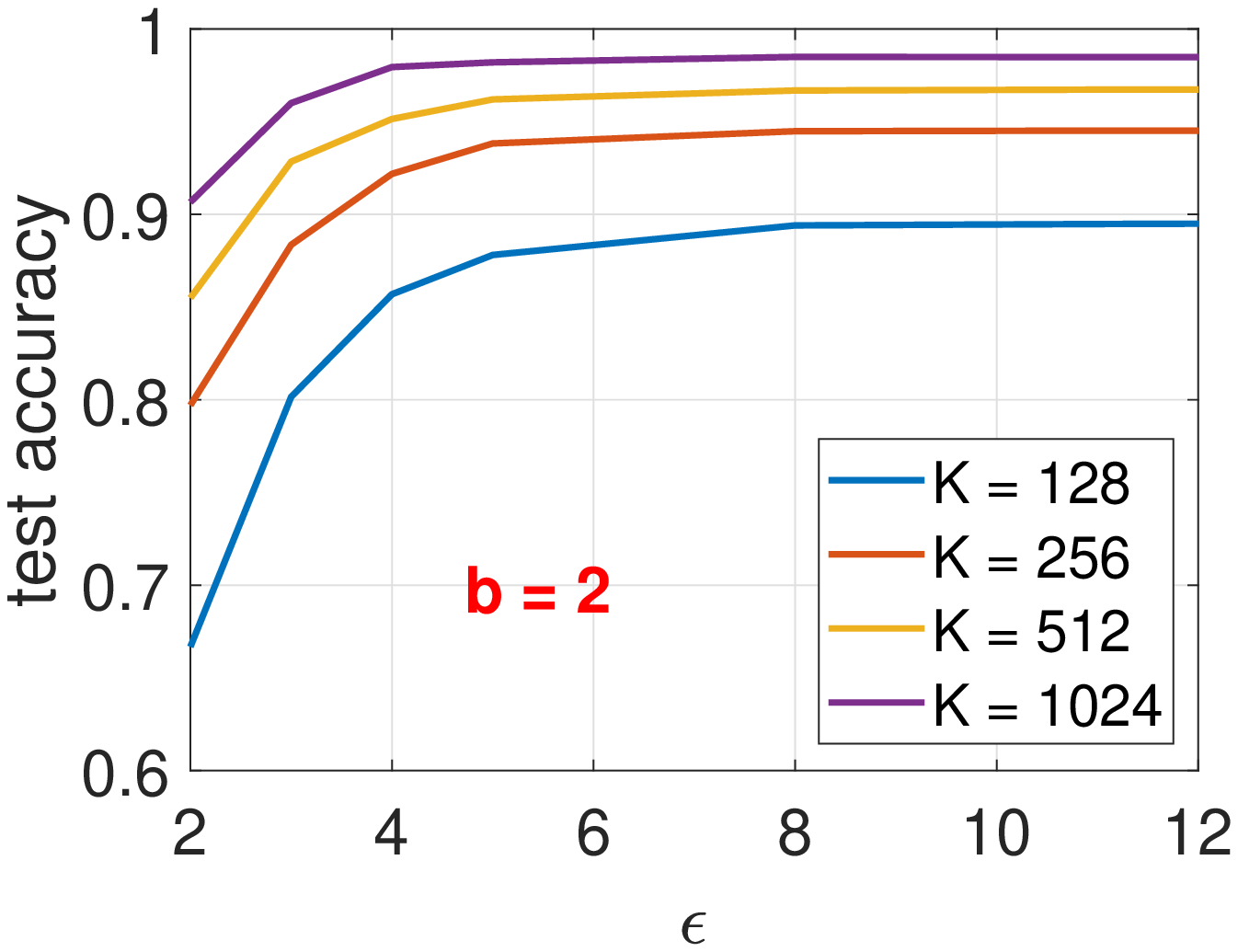} \hspace{-0.15in}
    \includegraphics[width=2.3in]{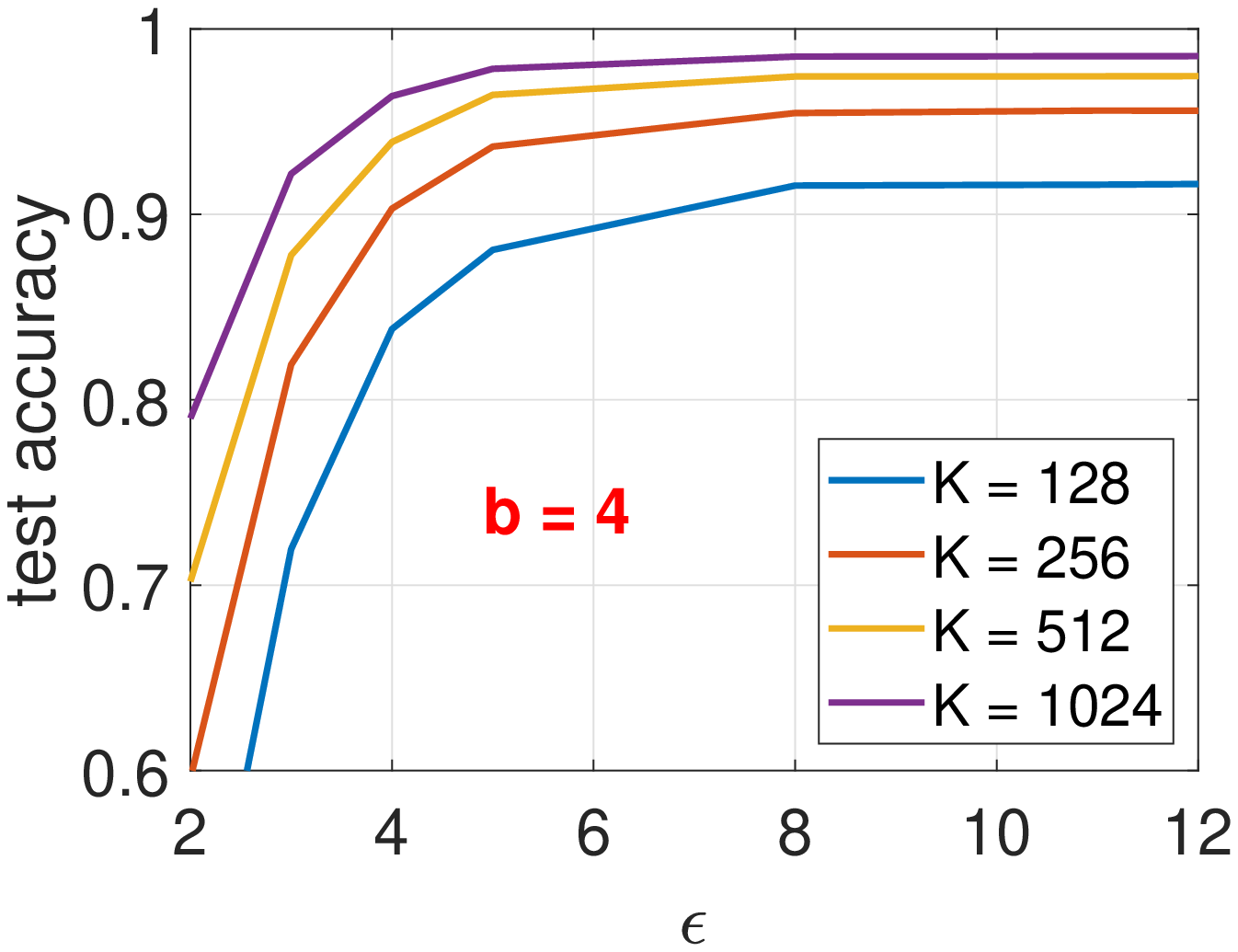}
    }

%\vspace{-0.15in}

\caption{Test classification accuracy of DP-BCWS on DailySports dataset~\citep{Asuncion+Newman:2007} with $l_2$-regularized logistic regression.}
\label{fig:BCWS_dailysports}
\end{figure}

In Figure~\ref{fig:BCWS_MNIST_MLP}, we train a neural network with two hidden layers of size 256 and 128 respectively on the MNIST dataset~\citep{lecun1998mnist}. We use the ReLU activation function and the standard cross-entropy loss. We see that, in a reasonable privacy regime (e.g., $\epsilon<10$), DP-BCWS is able to achieve $\approx 95\%$ test accuracy with proper $K$ and $b$ combinations (one can choose the values depending on practical scenarios and needs). For example, with $b=4$ and $K=128$, DP-BCWS achieves $\approx 97\%$ accuracy at $\epsilon=8$.

\begin{figure}[t]

    \mbox{\hspace{-0.15in}
    \includegraphics[width=2.3in]{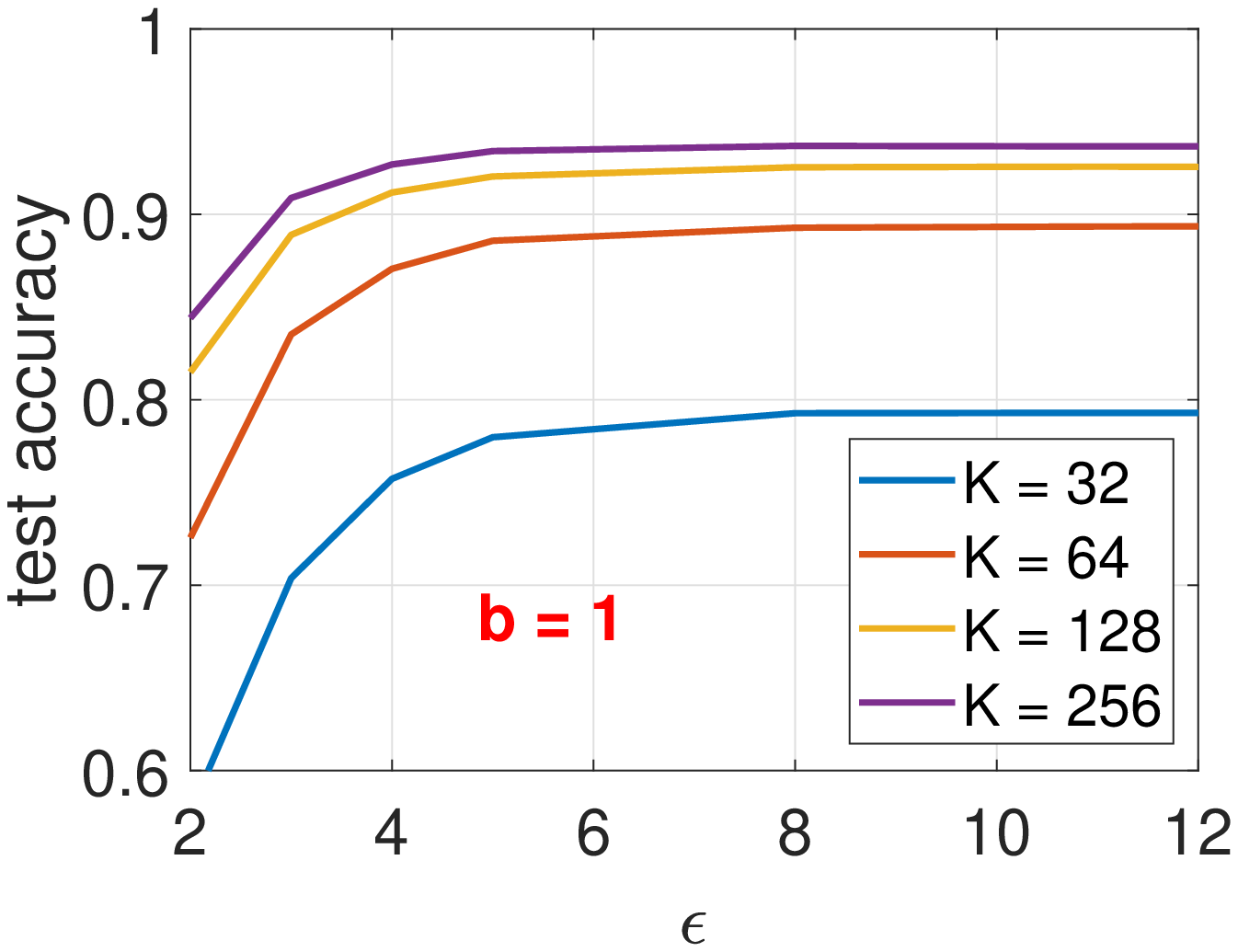} \hspace{-0.15in}
    \includegraphics[width=2.3in]{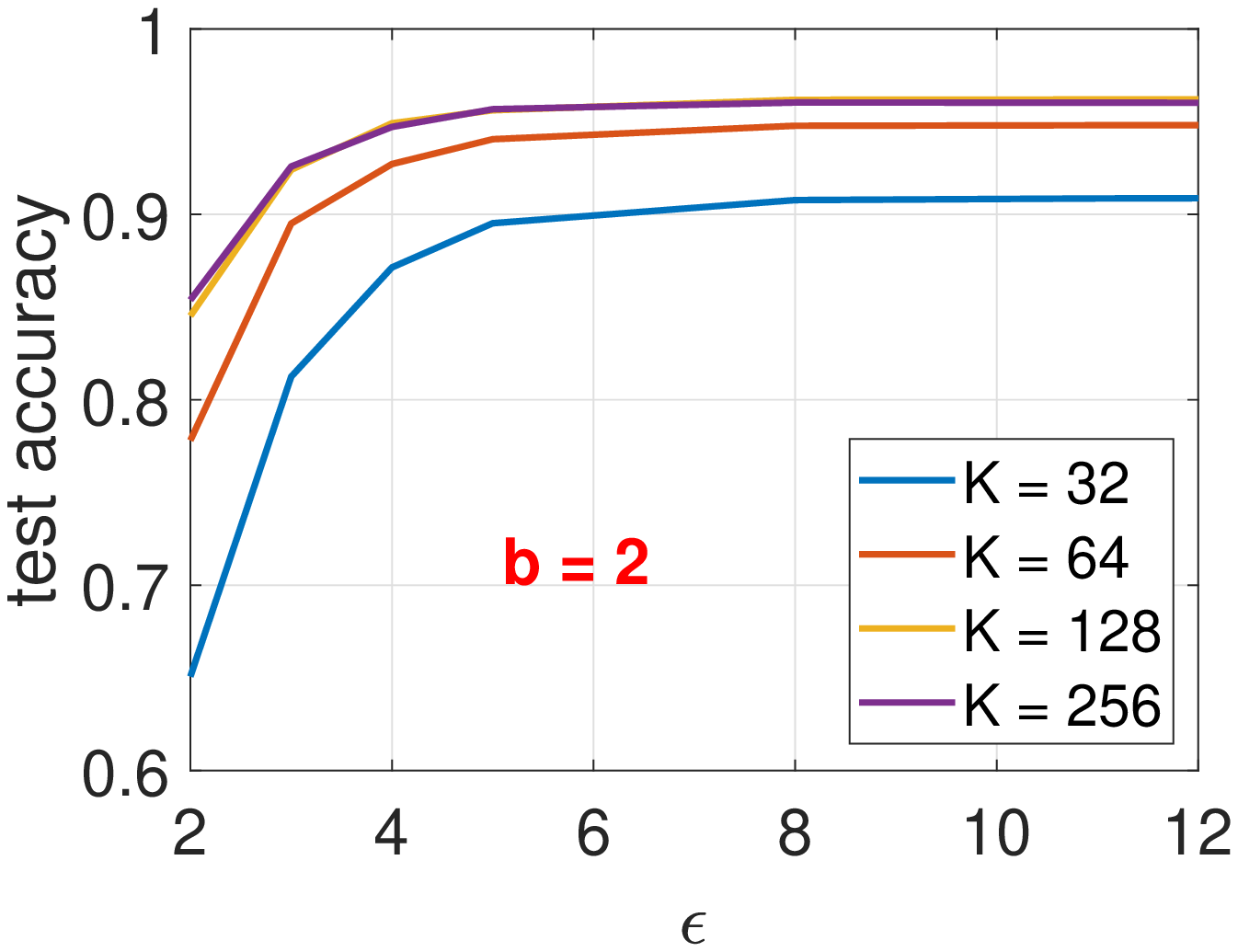} \hspace{-0.15in}
    \includegraphics[width=2.3in]{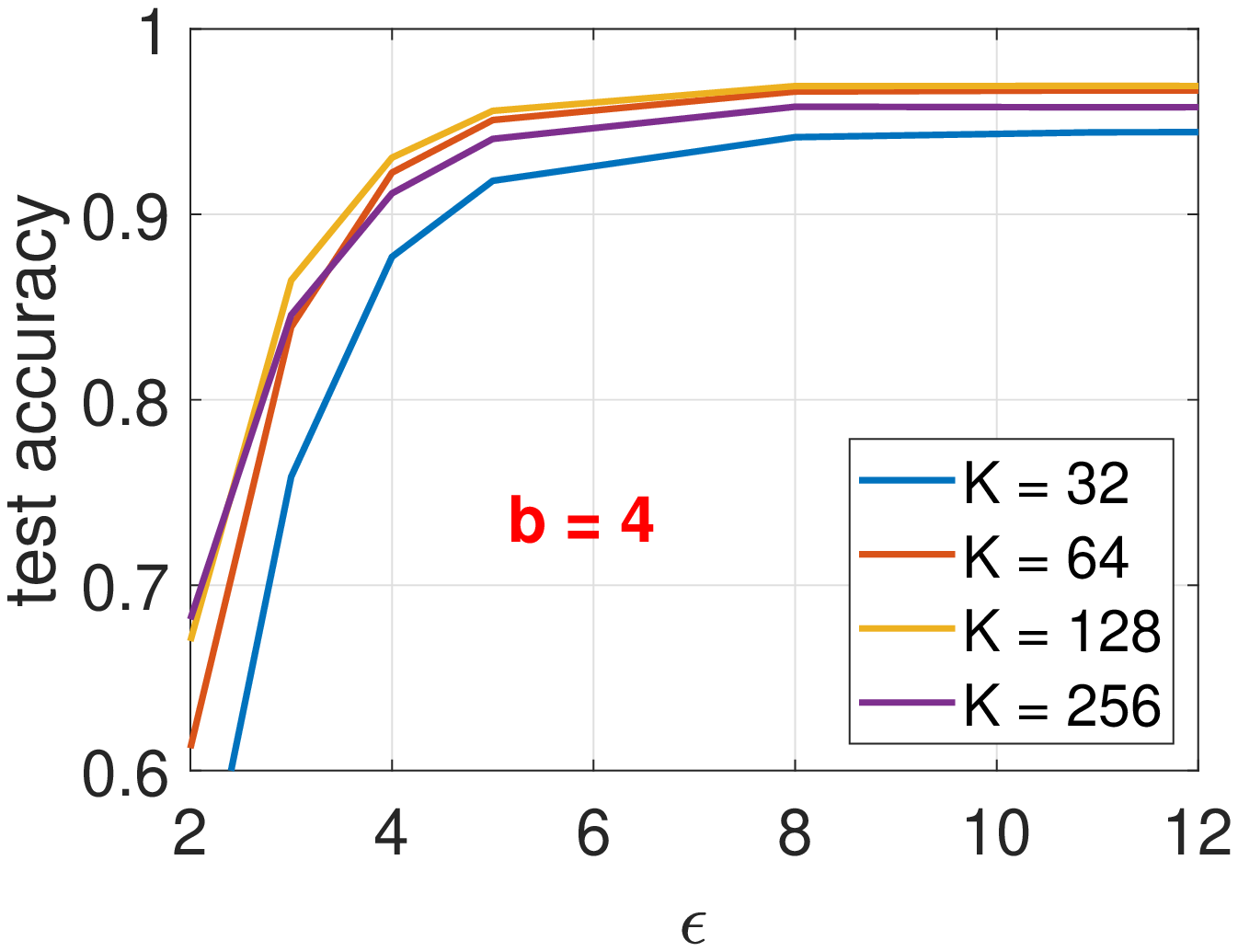}
    }

\vspace{-0.15in}

\caption{Test classification accuracy of DP-BCWS on MNIST dataset~\citep{lecun1998mnist} with 2-hidden layer neural network.}
\label{fig:BCWS_MNIST_MLP}\vspace{-0.05in}
\end{figure}

%\vspace{0.1in}

In a very recent survey paper~\citep[page 33]{ponomareva2023dp}, the authors provided their three-tier taxonomy of the strength of privacy protection,  according to the $\epsilon$ values: $\epsilon\leq 1$ for ``Tier 1: Strong formal privacy guarantees'', $\epsilon\leq 10$ for ``Tier 2: Realistic privacy guarantees'', and $\epsilon>10$ for ``Tier 3: Weak to no formal privacy guarantees''. Hence, our proposed algorithms are able to achieve good utility while providing strong realistic privacy protections.

\section{Conclusion}

Industrial applications often encounter massive high-dimensional binary (0/1) data. The method of ``one permutation hashing'' (OPH), which improves the standard minwise hashing (MinHash), is a common hashing technique for efficiently computing the binary Jaccard similarity. For example, it is a standard practice to use MinHash or OPH for building hash tables to enable sublinear-time approximate near neighbor (ANN) search. In this paper, we propose differentially private  one permutation hashing (DP-OPH). We develop three variants of DP-OPH depending on the densification procedure of OPH, and provide detailed derivation and privacy analysis of our algorithms. We analytically demonstrate the advantage of the proposed DP-OPH over DP-MH (DP for MinHash). Experiments  on retrieval and classification tasks  justify the effectiveness of the proposed DP-OPH, and provide guidance on the appropriate choice of the DP-OPH variant in different scenarios. 

\vspace{0.2in}

\noindent Our idea is extended to non-binary data and the weighted Jaccard similarity. We propose the DP-BCWS algorithm for non-binary data and empirically demonstrate its superb performance in utility for classification tasks with logistic regression models as well as neural nets. For example, we observe DP-BCWS achieves a good utility at $\epsilon<4\sim 8$. Given the high efficiency of the OPH method and its strong performance in privacy protection, we expect DP-OPH as well as DP-BCWS will be a useful large-scale hashing tool with privacy protection.

\vspace{0.2in}

\noindent In a recent work, \citet{li2023differential} utilized (smoothed) local sensitivity to develop  DP and iDP (individual DP) algorithms for the family of random projections (RP) and sign random projections (SignRP), with a series of algorithms including DP-OPORP, DP-SignOPORP, iDP-SignRP. Here ``OPORP''~\citep{li2023oporp} stands for ``one permutation + one random projection'' which is an improved variant of the count-sketch.  For our future work, we will seek novel new constructions to obtain the  analogous notion of local sensitivity to improve DP-OPH.

\bibliography{refs_scholar}
\bibliographystyle{plainnat}

\end{document}